\documentclass[12pt, letterpaper, twoside]{article}
\usepackage{cite}
\usepackage{amsmath,amssymb,amsfonts}
\usepackage{amsthm}
\usepackage{hyperref} 
\usepackage{algorithmic}
\usepackage{graphicx}
\usepackage{textcomp}
\usepackage[usenames,dvipsnames,table]{xcolor}
\usepackage{mathtools}
\usepackage{xspace}
\usepackage{microtype}
\usepackage{tkz-graph}
\usepackage[USenglish]{babel}
\usepackage[ruled,vlined,linesnumbered,noend,algo2e]{algorithm2e} 
\usepackage{caption}
\usepackage{subcaption} 
\usepackage{array}
\usepackage{tkz-graph}
\usepackage{authblk}
\usepackage{nag}       % Issues warnings when best practices in writing LaTeX

\usepackage{booktabs}   % Improves the typesettings of tables.
\usepackage{enumitem}

\newcommand{\edge}[2]{\{#1, #2\}}
\newcommand{\sset}[1]{\mathcal{#1}}

\DeclareMathOperator{\smt}{smt}
\DeclareMathOperator{\csmt}{csmt}
\DeclareMathOperator{\computeSMT}{retraceSMT}
\DeclareMathOperator{\retraceTree}{tree}

\newcommand{\Card}[1]{\left|#1\right|}

\newcommand{\SB}{\{}%
\newcommand{\SM}{\mid}%
\newcommand{\SE}{\}}%
\def\hy{\hbox{-}\nobreak\hskip0pt}

\newcommand{\eqdef}{\ensuremath{\,\mathrel{\mathop:}=}}

\newcommand{\defeq}{\vcentcolon=}
\newcommand{\astar}{\ensuremath{\text{A}\hspace{-0.125em}^*}\xspace}
\newcommand{\dsstar}{\ensuremath{\text{DS}\hspace{-0.125em}^*}\xspace}
\newcommand{\dsstarsolver}{\ensuremath{\text{DS}\hspace{-0.125em}^*\allowbreak{}Solve}\xspace}    
\newcommand{\longversion}[1]{}

\newcommand{\newacronym}[3]{}
\newacronym{sph}{SPH}{Shortest Path Heuristic}
\newacronym{rsph}{RSPH}{Repeated Shortest Path Heuristic}
\newacronym{sap}{SAP}{Steiner Arborescence Problem}
\newacronym{pace}{PACE}{Parameterized Algorithms and Computational Experiments Challenge}
\newacronym{ptm}{PTm}{Paths with many Terminals}
\newacronym{sdc}{SDC}{Steiner Distance Circuit}
\newacronym{ntdk}{$NTD_k$}{Non-Terminals of degree k}
\newacronym{ilp}{ILP}{Integer Linear Program}
\newacronym{lp}{LP}{Linear Program}
\newacronym{smt}{SMT}{Steiner minimal tree}
\newacronym{fpt}{FPT}{Fixed Parameter Tractable}
\newacronym{tsp}{TSP}{Traveling Salesman Problem}
\newacronym{mst}{MST}{Minimum Spanning Tree}
\newacronym{stp}{STP}{Steiner tree problem}
\newacronym{bb}{B\&B}{Branch and Bound}
\newacronym{bc}{B\&C}{Branch and Cut}
\newacronym{mcn}{MCNFP}{Minimum-cost Network Flow Problem}
\newacronym{nsv}{NSV}{Nearest Special Vertices}
\newacronym{dwa}{DWA}{Dreyfus-Wagner Algorithm}
\newacronym{dsa}{DSA}{Dijkstra-Steiner algorithm}
\newacronym{dsaa}{DSAA}{Dijkstra-Steiner Algorithm for Admissibility}
\newacronym{mips}{MIPS}{Mixed Integer Programming Solver}

\newtheorem{thm}{Theorem}
\newtheorem{Theorem}[thm]{Theorem}%[theorem]
\newtheorem{Example}[thm]{Example}
\newtheorem{Definition}[thm]{Definition}%[theorem]
\newtheorem{proposition}[thm]{Proposition}%[theorem]
\newtheorem{lemma}[thm]{Lemma}%[theorem]

\makeatletter
\newcommand{\algorithmfootnote}[2][\footnotesize]{
  \let\old@algocf@finish\@algocf@finish
  \def\@algocf@finish{\old@algocf@finish
    \leavevmode\rlap{\begin{minipage}{\linewidth}
    #1#2
    \end{minipage}}
  }
}
\makeatother

\newcommand{\footnoteitext}[1]{\stepcounter{footnote}
  \footnotetext[\thefootnote]{#1}}

\newcolumntype{H}{>{\setbox0=\hbox\bgroup}c<{\egroup}@{}}

\SetKwProg{Fn}{Function}{is}{}
\SetKwRepeat{Do}{do}{while}

\SetKwInput{KwData}{In}
\SetKwInput{KwResult}{Out}
\SetAlFnt{\small}
\SetAlCapFnt{\small}
\SetAlCapNameFnt{\small}
\SetAlCapHSkip{0pt}
\SetEndCharOfAlgoLine{}

\DontPrintSemicolon

\newenvironment{restatetheorem}[1][\unskip]{%
  \begingroup
  \def\thetheorem{\ref{#1}}
}%
{%
  \endgroup
}%

\begin{document}

\title{Solving the Steiner Tree Problem\\ with few Terminals}

\author[1]{Fichte Johannes K.}
\author[2]{Hecher Markus\footnote{Hecher is also affiliated with the University of Potsdam, Germany.}}
\author[2]{Schidler Andr\'{e}}
\affil[1]{TU Dresden}
\affil[2]{TU Wien}

\maketitle

\begin{abstract}
The Steiner tree problem is a well-known problem in network design, routing, and VLSI
  design. Given a graph, edge costs, and a set of dedicated vertices (terminals), the Steiner tree problem asks to output a sub-graph that connects all  terminals at minimum cost.
%
% a connects all the terminals with minimum total edge cost
%
% Steiner tree,  which is a spanning Tree .
%
% One state-of-the-art dynamic programming approach to the Steiner tree problem is the Dijkstra-Steiner algorithm. The algorithm builds a Steiner tree by strategically combining Steiner trees for smaller instances, created by using only a subset of the terminals. This search is performed in a greedy manner and relies heavily on a guiding heuristic function.
%
A state-of-the-art algorithm to solve the Steiner tree problem by means of dynamic programming is the Dijkstra-Steiner algorithm.
The algorithm builds a Steiner tree of the entire instance by systematically searching for smaller instances, based on subsets of the terminals, and combining Steiner trees for these smaller instances. The search heavily relies on a guiding heuristic function in order to prune the search space.
However, to ensure correctness, this algorithm allows only for limited heuristic functions, namely, those that satisfy a so-called consistency condition.

In this paper, we enhance the Dijkstra-Steiner algorithm and establish a revisited algorithm, called \dsstar.
The \dsstar algorithm allows for arbitrary lower bounds as heuristics relaxing the previous condition on the heuristic function. 
Notably,  we can now use linear programming based lower bounds. 
Further, we capture new requirements for a heuristic function in a condition, which we call admissibility. We show that admissibility is indeed weaker than  consistency and establish correctness of the \dsstar algorithm when using an admissible heuristic function. 
We implement \dsstar and combine it with modern preprocessing, resulting in an open-source solver (\dsstarsolver). Finally, we compare its performance on standard benchmarks and observe a competitive behavior.% to state-of-the-art solvers.
\end{abstract}

%\begin{IEEEkeywords}
%\FIX{Keywords}
%component, formatting, style, styling, insert
%\end{IEEEkeywords}

\section{Introduction}
% The term Steiner tree problem refers to a class of graph problems with the common objective of connecting a specified set of vertices, the so called \emph{terminals}, as cheaply as possible. This problem occurs frequently and therefore numerous different definitions are used in practice~\cite{HauptmannKarpinski15a}. In
% this paper, we consider the \emph{minimum Steiner tree problem (STP)}. Here, edge costs are given in addition to graph and terminals. The goal is to connect the terminals in a sub-graph where summed up costs of the edges is minimal~\cite{HwangRichardsWinter92a}\FIX{What to cite? There are older papers on the STP}. This sub-graph is called a \emph{Steiner tree}.

% Given a graph and a subset of its vertices, so-called
% \emph{terminals}, the Steiner tree problem asks to connect the
% terminals at certain notion of minimum cost.  In practice, numerous
% versions and different definitions exist~\cite{HauptmannKarpinski15a}.

The term \emph{Steiner tree problem on graphs} encompasses a class of graph problems that
all ask to connect specific vertices of a graph, so-called
\emph{terminals}, at minimum cost.  Connecting
the vertices usually requires to select edges of a given graph that form 
a tree and the resulting solution is called a \emph{Steiner tree}.
While numerous different definitions
exist~\cite{HauptmannKarpinski15a}, we consider the \emph{minimum Steiner tree problem (STP)}, where costs are defined by positive integers given at the edges. The total cost  is simply the sum of costs of selected edges and one aims to minimize the total  cost~\cite{HwangRichardsWinter92a}.

STP is long known to be computationally hard,~i.e., NP-hard~\cite{Karp72}. 
Despite the longevity and the large body of existing research, solving STP is still an active research topic~\cite{FischettiEtAl17a,HougardySilvanusVygen16a,IwataShigemura19a,Nederlof13a,
PajorUchoaWerneck18a}. 
Interest stems from applications in various fields such as the
construction of evolutionary trees in phylogeny~\cite{Gusfield97a}, 
network design~\cite{MagnantiWong84a}, routing
problems~\cite{Gilbert89a}, and VLSI
design~\cite{GrotschelMartinWeismantel97a}.
Over time, many heuristic and exact solving algorithms and techniques
have been developed~\cite{HwangRichardsWinter92a}. 
One of the longest known algorithms for STP is the
Dreyfus-Wagner algorithm~\cite{DreyfusWagner72a}. 
This
algorithm implements dynamic programming with runtime 
bounded exponentially in the number
of terminals and polynomially in the number of vertices.
%
% Such algorithms are known as fixed parameter
% tractable~\cite{CyganEtAl15} and are theoretically fast, as long as the number of terminals remains small\FIX{this is a little inaccurate, accurate would be, as long as the parameter, the terminals, is small; MH: I can live with inaccuracy :)}. 
%
%
In theory, the Dreyfus-Wagner algorithm is fast as long as the
number of terminals remains small.
However, in practice, the involved constants are far from optimal and
the runtime of the algorithm increases too~quickly
% For that reason, these DP algorithms can only
% be used for very specific instances.
%
%
as the algorithm exhaustively enumerates all sub-solutions. 
This can be avoided and the runtime improved by replacing 
enumeration with graph search, resulting in the Dijkstra-Steiner
algorithm~\cite{HougardySilvanusVygen16a}.
Its underlying idea is similar to the well-known \astar algorithm~\cite{HartNilssonRaphael68a}. One uses the Dijkstra's algorithm to navigate the search space and guide the search using a heuristic function.
In practice, the heuristic function is crucial for the performance of the algorithm.
Unfortunately, the heuristic needs to satisfy a strong condition, namely, providing consistent lower bounds, as otherwise correctness cannot be assured. Consistency guarantees that the lower bound is monotonous and the estimate will not decrease in later iterations of the Dijkstra-Steiner algorithm.

\medskip\noindent\textbf{Contributions.}  In this paper, we revisit
the Dijkstra-Steiner algorithm and enable heuristic functions for
general lower bounds relaxing the previously known consistency condition.  
We formalize
new requirements in a condition called \emph{admissibility}.  We prove
that admissibility is weaker than consistency and allows for better
heuristic functions. 
In particular, we can now employ linear programming based approximations as heuristics.
We show that our approach still ensures correctness.
We implement the new algorithm into a fully fledged solver,
\emph{\dsstarsolver}, for STP.
Our solver complements the solving algorithm with modern preprocessing, upper bound heuristics, and local optimization.
Finally, we present experimental results where we compare our solver
 to state-of-the-art STP solvers.
Our experiments show that our revisited algorithm with the new heuristic
significantly improves the runtime.
%
% and the best solvers from the PACE~2018 challenge~\cite{BonnetSikora19a} on instances from the
% standard benchmark sets SteinLib, Copenhagen, Vienna and PACE~2018.
% %

\medskip\noindent\textbf{Algorithms and latest implementations.}  The
Dreyfus-Wagner algorithm~\cite{DreyfusWagner72a} was generalized by
Erickson, Monma, and Veinott~\cite{EricksonMonmaVeinott87a}.
While the algorithm is exponential in the number of terminals with a
basis of~3, it was very recently implemented into the solver Pruned
showing good performance on certain
benchmarks~\cite{IwataShigemura19a}.
%
% The base can be decreased to almost~2 at the cost of polynomial factors of degree 100 and higher~\cite{FuchsEtAl07b}\FIX{Removing this sentence would also remove one reference}.
The Dijkstra-Steiner algorithm~\cite{HougardySilvanusVygen16a}  is implemented in  the solvers  Jagiellonian~\cite{MaziarzPolak18a} and HSV~\cite{HougardySilvanusVygen16a}. We also implemented it  into our solver to measure the difference with our revisited algorithm and heuristic.
%
% There are also dynamic programming algorithms tackles instances of low treewidth.
% While a standard textbook algorithm runs in time~$w^{\mathcal{O}(w)}$ 
% where $w$ is the width of the given tree
% decomposition~\cite{CyganEtAl15}, an improved algorithm is
% rank-based~\cite{BodlaenderEtAl15a}.
%
While our approach tackles instances with a limited number of
terminals, there are also dynamic programming algorithms that employ
low treewidth or rank~\cite{CyganEtAl15,BodlaenderEtAl15a}. The
solvers wata\_sigma~\cite{Iwata18a}, Tom~\cite{Zanden18a},
and~FIT~CTU~\cite{MituraSuchy18a} implement such algorithms and showed
successful results in the PACE~2018 challenge.
Another approach is branch-and-cut, which is based on
linear programming using existing solvers for (mixed) integer linear
programs (MILP).
% , which is a variation of branch-and-bound solvers
%
Various solvers are available that perform very well on arbitrary STP
instances and can mainly be distinguished by the used linear
programming model.
Two notable solvers are mozartballs~\cite{FischettiEtAl17a} and
SCIP-Jack~\cite{GamrathEtAl17a}.

\noindent\textbf{Challenges.}
In 2014, the 11th DIMACS implementation challenge was dedicated to different variants of the Steiner tree problem~\cite{JohnsonEtAl14a}. The 3rd
Parameterized Algorithms and Computational Experiments challenge (PACE
2018) addressed the minimum Steiner tree problem and featured three
tracks, namely, a) where instances were limited in the number
of terminals, b) where instances were limited in the treewidth, and c)
that allowed to submit heuristics~\cite{BonnetSikora19a}. Solvers in these competitions mainly used  two different techniques: \emph{dynamic programming} and
\emph{branch-and-cut}.
%

%

%\todo{Rename our solver (double blind requirement). DS*}

%\todo{Pls remove the awkaward macros and try to use as few as possible Macros}

%\todo{Use the bibtex style; Markus and I use}
%\todo{anonymous repo on preferably dropbox + tinyurl, remove names and header if possible}

%
% \begin{figure*}[t]
% %\begin{minipage}{.5\textwidth}
% %\captionsetup{width=.95\linewidth}
% \centering
% \begin{subfigure}{0.6\columnwidth}
% \centering
% \resizebox{.75\columnwidth}{!} {%
% 	\input{graphs/base_graph.tikz}%
% }%
% \subcaption{A network~$N$ (running example).}%
% \label{fig:graph-base}%
% \end{subfigure}%
% %
% \begin{subfigure}{0.7\columnwidth}
% \centering
% \resizebox{.45\columnwidth}{!} {%
% \input{graphs/dist_graph.tikz}%
% }%
% \subcaption{Distance network~$D_N(\{a,b,c,d\})$.}%
% \label{fig:graph-dist}
% \end{subfigure}%
% %
% \begin{subfigure}{0.7\columnwidth}
% \centering
% \resizebox{.54\columnwidth}{!} {%
%     \input{graphs/smt_graph.tikz}%
% }
% \subcaption{An SMT of~$(N,\{a,b,c,d\})$, cf., Figure~\protect\ref{fig:graph-base}.}
% \label{fig:graph-smt}
% \end{subfigure}%
% %
% \vspace{-.5em}
% \caption{A network~$N$ (left), and a distance network (middle) as well as an SMT (right) thereof.}
% \end{figure*}

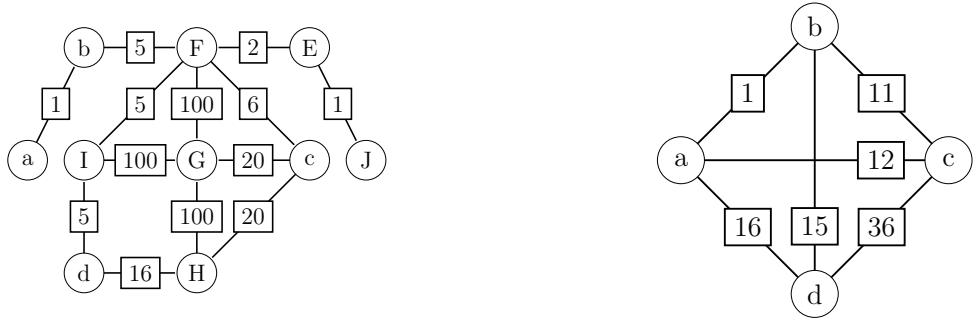
\begin{figure*}[t]
%\begin{minipage}{.5\textwidth}
%\captionsetup{width=.95\linewidth}
\centering
\begin{subfigure}{0.5\columnwidth}
\centering
\resizebox{.75\columnwidth}{!} {%
	\begin{tikzpicture}[shorten >=1pt]
	\definecolor{clr1}{RGB}{166,206,227}
	\definecolor{clr2}{RGB}{31,120,180}
	\definecolor{clr3}{RGB}{178,223,138}
    \definecolor{clr4}{RGB}{51,160,44}
     \SetGraphUnit{2}% modifie la distance entre les nodes
     \GraphInit[vstyle=normal]
     \tikzset{LabelStyle/.style =   {draw,
                                     %fill  = yellow,
                                     %text  = red
                                     }}
     \tikzset{VertexStyle/.style = {shape          = circle,
     							draw,
                                 minimum size   = 20 pt,
                                 inner sep      = 2pt,
                                 outer sep      = 0pt
                                 }}
    \Vertex{F}
    \SOWE(F){I} 
    \EA(I){G}
    \SO(I){d}
    \EA(d){H}
    \EA(G){c}
    \WE(F){b}
    \EA(F){E}
    {
        \SetGraphUnit{1}
        \WE(I){a}
        \EA(c){J}
    }
         
    %\tikzset{EdgeStyle/.append style = {bend left}}
    \Edge[label=1](a)(b)
    \Edge[label=2](E)(F)
    \Edge[label=5](b)(F)
    \Edge[label=6](F)(c)
    \Edge[label=100](F)(G)
    \Edge[label=100](I)(G)
    \Edge[label=100](H)(G)
    \Edge[label=20](c)(G)
    \Edge[label=20](c)(H)
    \Edge[label=16](d)(H)
    \Edge[label=5](d)(I)
    \Edge[label=5](F)(I)
    \Edge[label=1](E)(J)
\end{tikzpicture}%
}%
%\subcaption{A network~$N$ (running example).}%
%\label{fig:graph-base}%
\end{subfigure}%
\begin{subfigure}{0.7\columnwidth}
\centering
\resizebox{.45\columnwidth}{!} {%
\begin{tikzpicture}
     \SetGraphUnit{2}% modifie la distance entre les nodes
     \GraphInit[vstyle=normal]
     \tikzset{LabelStyle/.style =   {draw,
                                     %fill  = yellow,
                                     %text  = red
              	                      }}
      \tikzset{VertexStyle/.style = {shape          = circle,
     							draw,
                                 minimum size   = 20 pt,
                                 inner sep      = 2pt,
                                 outer sep      = 0pt
                                 }}
     \Vertex{b}
     \SOEA(b){c}
     \SOWE(b){a}
     \SOEA(a){d}

     %\tikzset{EdgeStyle/.append style = {bend left}}
     \Edge[label=1](a)(b)     
     \Edge[label=16](a)(d)
     \Edge[label=11](b)(c)
     \Edge[label=36](c)(d)   
     \tikzset{LabelStyle/.style =   {draw, pos=0.8}}
     \Edge[label=12](a)(c)
     \Edge[label=15](b)(d)
\end{tikzpicture}%
}%
%\subcaption{Distance network~$D_N(\{a,b,c,d\})$.}%
%\label{fig:graph-dist}
\end{subfigure}%
\begin{subfigure}{0.7\columnwidth}
\centering
\resizebox{.54\columnwidth}{!} {%
    \begin{tikzpicture}
     \SetGraphUnit{2}% modifie la distance entre les nodes
     \GraphInit[vstyle=normal]
     \tikzset{LabelStyle/.style =   {draw,
                                     %fill  = yellow,
                                     %text  = red
                                     }}
      \tikzset{VertexStyle/.style = {shape          = circle,
     							draw,
                                 minimum size   = 20 pt,
                                 inner sep      = 2pt,
                                 outer sep      = 0pt
                                 }}
     \Vertex{I}
     \NOEA(I){F}
     \SO(I){d}
     \SOEA(F){c}
     \WE(F){b}
    {
     \SetGraphUnit{1}
     \WE(I){a}
    }
         
     %\tikzset{EdgeStyle/.append style = {bend left}}
     \Edge[label=1](a)(b)
     \Edge[label=5](b)(F)
     \Edge[label=6](F)(c)   
     \Edge[label=5](d)(I)
     \Edge[label=5](F)(I)
\end{tikzpicture}%
}
%\subcaption{An SMT of~$(N,\{a,b,c,d\})$, cf., Figure~\protect\ref{fig:graph-base}.}
%\label{fig:graph-smt}
\end{subfigure}%
%
%\vspace{-.5em}
\caption{Example of a network~$N$ (left), distance network~$D_N(\{a,b,c,d\})$ (middle), and an SMT~$(N,\{a,b,c,d\})$ (right).}
\label{fig:network}
\end{figure*}

\section{Preliminaries}\label{chap:prelims}
\longversion{%
%In this section, we establish the notations and background knowledge required throughout this thesis. We start by introducing some \emph{complexity theory} basics. We then discuss the necessary \emph{graph theory} definitions. This allows us to define the \emph{\glsfirst{stp}}. Finally, we introduce \emph{linear programming}.
}

%We will give a quick overview of the definition required for this thesis. We assume a general knowledge about \emph{Turing machines} and \emph{languages}.

We assume familiarity with standard notions in computational
complexity~\cite{Papadimitriou94} and let 
$\mathbb{N}$ be the set of positive integers.

\medskip\noindent\textbf{Graphs, Networks, and Steiner Trees.}\label{chap:prelims-graphs}
\label{chap:prelims-st}
For basic terminology on graphs, we refer to the literature~\cite{Diestel12}. %,BondyMurty08}\FIX{Does one suffice?}.
In particular, we use $G = (V, E)$ to denote an \emph{undirected, connected graph}, or graph for short,
where~$V$ is a set of \emph{vertices} and~$E\subseteq \{\{u,v\} \mid u,v\in V\}$ is a 
set of \emph{edges}. 
Given graphs~$G_1=(V_1,E_1)$ and $G_2=(V_2,E_2)$, $G_1$ is a \emph{sub-graph} of~$G_2$ if~$V_1\subseteq V_2$ and $E_1\subseteq E_2$.
Given a graph~$G=(V,E)$ and a vertex~$u\in V$.
%For a vertex~$v$, we denote 
Then, the set~$\delta_G(u)$ of \emph{incident edges} is given by $\delta_G(u)\eqdef \{\{u,v\} \mid \{u,v\} \in E\}$ and~$\Card{\delta_G(u)}$ refers to the \emph{degree of~$u$}.
%
%
% A \emph{path} is a graph with $k + 1$ pairwise distinct vertices
% $v_1$,$\ldots$,$v_{k+1}$, and k distinct edges~$\{v_i,v_{i+1}\}$ where
% $1 \leq i \leq k$ (possibly k = 0).  A \emph{cycle} of length $k$ is a
% graph that consists of $k$ distinct vertices
% $v_1$,$v_2$,$\ldots$,$v_k$ and $k$ distinct edges \{v_1, v_2\},
% \{v_2,v_3\}, $\ldots$,$\{v_{k−1},v_{k}\}$,$\{v_{k},v_1\}$. A graph is
% acyclic if it does not contain a cycle as subgraph.
% %
% A connected component $C$ of $G$ is an inclusion-maximal subgraph
% $C = (V_C , E_C )$ of $G$ such that for any two vertices
% $u, v \in V_C$ there is a path in $C$ from $u$ to $v$.
% % We say $G$ is a tree if it is a connected component $C = G$ and $G$
% % contains no cycles.
% %
%
% A \emph{path} is a sequence with $k + 1$ pairwise distinct vertices
% $v_1$,$\ldots$,$v_{k+1}$, and k distinct edges~$\{v_i,v_{i+1}\}$ where
% $1 \leq i \leq k$ (possibly k = 0).
%
A \emph{tree} is a graph that is acyclic and connected, i.e., for every pair of vertices of the graph there is a path between them.
%
%A \emph{directed graph}~$(V,A)$ consists of set~$V$ of vertices 
%with~$A\subseteq V\times V$ being a set of \emph{arcs}.
%
%TODO: This was a "graph network" before, do we really call it graph network? Then we need to either establish "network" as well, or write "graph network" everywhere. (asc)
%\todo{Is it a problem that before A was used for directed edges?}
%
Further, we define an \emph{undirected network}, or network for short, by~$N\eqdef(V, E, \sigma)$, 
	where graph~$(V,E)$ is connected and
%and~
$\sigma: E \rightarrow \mathbb{N}$ is a total mapping, called
\emph{edge costs}, which assigns to each edge $e$ some integer, called
\emph{cost} of $e$.
Let $N=(V,E,\sigma)$ be a network and let
%
% $d_N(u,v)\eqdef \min \SB c_N(P) \SM P = \{u,v_1\}, \ldots, \{v_k,v\}
% \text{ is a path}\SE$.
%
%, where~$P$ is a graph that forms a path between $u$ and $v$.
%of the shortest path (in regard to $\sigma$) from $u$ to $v$. 
%
$G=(V',E')$ be a sub-graph of~$(V,E)$.
We let the \emph{costs} of $G$ in $N$ be 
$c_N(G) \eqdef \sum_{e \in E'}\sigma(e)$.
Then, we denote by $d_N(u, v)$ the
\emph{distance} between two vertices~$u \in V$ and $v \in V$ in
$N$,~i.e.,
$d_N(u,v)$ is the length of the shortest path from~$u$ to $v$.
Formally, $d_N(u,v)\eqdef \min \SB c_N(P) \SM P\text{ is a connected sub-graph of~$G$ that contains } u \text{ and } v\SE$.
Further, we define the \emph{distance network} of~$V'$ for~$N$ by
%\todo{Either $V' \times V'$ or $\SB \{u,v\} \SM u,v \in V', u \neq v\SE$}
$D_N(V')\eqdef(V', \SB \{u,v\} \SM u,v \in V'\SE, \sigma')$,
where~$\sigma'(\{u,v\})\eqdef d_N(u,v)$ for $u,v\in V'$.

%\todo{dsa vs dsaa}

\begin{Example}\label{ex:running}
  Figure~\ref{fig:network} (left) %\ref{fig:graph-base}
  illustrates network~$N$, which we use throughout the paper. Circles
  represent vertices, lines are edges, and edge costs are given by the
  rectangles. Further, terminals are denoted by lowercase letters.
  Figure~\ref{fig:network} (middle) %\ref{fig:graph-dist}
  provides the distance network $D_N(\{a, b, c, d\})$ for~$N$.
\hfill
\end{Example}

An instance~$\mathcal{I}$ of the \emph{minimum Steiner Tree Problem (STP)} is
of the form~$\mathcal{I}\eqdef(N,R)$, where~$N=(V,E,\sigma)$ is a
network and $R\subseteq V$ is a non-empty set of vertices,
called \emph{terminals}.
%
%Let~$I=(N,R)$ be such an instance, where
%$N=(V,E,\sigma)$.
A \emph{Steiner tree} for~$\mathcal{I}$ is a rooted tree $T\eqdef(V', E', r)$,
where $V' \subseteq V$, $R \subseteq V'$, $E' \subseteq E$, and~$r\in R$. We call $r$ the \emph{root}.
$\mathcal{T}(\mathcal{I})$ refers to the \emph{set of Steiner trees} for~$\mathcal{I}$. 
The tree~$T$ is called an \emph{SMT (Steiner minimal tree)}, if $c_N(T)=\min_{T' \in \mathcal{T}(\mathcal{I})} c_N(T')$.
Let~$\smt(\mathcal{I})$ be the set of~SMTs for~$\mathcal{I}$
%\eqdef \{T \mid T\in \mathcal{T}(\mathcal{I}) \mid c_N(T)=\min_{T' \in \mathcal{T}(\mathcal{I})} c_N(T'))$
%be the set of Steiner trees of minimal costs.
%We refer to any~$T\in \smt(\mathcal{I})$ by \emph{SMT}
and let $\csmt(\mathcal{I})\eqdef c_N(T)$ for any~$T\in\smt(\mathcal{I})$, called the \emph{$\csmt$ (Steiner minimal tree costs)}.
%We call $T$ a \emph{SMT},
%if $c_N(T)=\min_{T' \in \mathcal{T}(\mathcal{I})} c_N(T')$.
%Further, given a set~$I\subseteq R$, 
%let~$\csmt(N, I)\eqdef c_N(T')$ for any~$T'\in\mathcal{T}((N,I))$.
%
The tree~$T$ is a \emph{minimum spanning tree} of network~$N$, %=(V,E,\sigma)$,
if~$T$ is an SMT for~$(N,V)$.
%For the sake of readability, we write $\delta, c, d, D$ instead of 
%$\delta_G, c_N, d_N, D_N$ if the subscript is clear from the context, respectively.

\begin{Example}\label{ex:instance}
Figure~\ref{fig:network} (right) %\ref{fig:graph-smt} 
depicts the SMT for our instance~$\mathcal{I}=(N,R)$ 
from Example~\ref{ex:running}, 
where~$N$ is given (left). % in Figure~\ref{fig:graph-base}. 
The set~$R$ of terminals consists of the vertices~$R=\{a,b,c,d\}$ 
in lower case letters (left). %of Figure~\ref{fig:graph-base}.
\end{Example}

\medskip\noindent\textbf{Dijkstra's algorithm and the $A^*$ algorithm.}\label{chap:dij-astar}
Given a network~$N=(V,E,\sigma)$.
\emph{Dijkstra's algorithm} is used to find the shortest path from vertex $u \in V$ to vertex~$v\in V$.
In the following, we briefly mention the ideas of this algorithm.
For details, we refer to the literature~\cite{dijkstra_note_1959}.
Dijkstra's algorithm maintains a queue $Q$ and for each vertex~$w\in V$ a distance $l(w)$ between vertices $u$ and $w$. 
Initially, $Q = \{u\}$ and the distances $l(\cdot)$ are assumed to be $\infty$\footnote{In this paper, we use $\infty$ as an abbreviation for~$\Sigma_{e\in E}\sigma(e)$.}, except for $l(u)$, where the distance is known to be $0$. 
In each iteration the algorithm removes the vertex~$w$ from~$Q$ that minimizes $l(w)$.
%that minimizes $l(v)$. 
Then, the algorithm \emph{expands}~$w$: 
For each $\{w, x\} \in \delta_{(V,E)}(w)$, the value~$l(w) + \sigma(\{w,x\})$ is computed. 
Whenever $l(w) + \sigma(\{w,x\}) < l(x)$, vertex $x$ is added to $Q$ and the (smaller) distance~$l(x) \leftarrow l(w) + \sigma(\{w,x\})$ is kept for~$x$. 
Whenever a node $w$ is removed from $Q$, it holds that $l(w) = d_N(u,w)$. Therefore, the algorithm stops as soon as~$v$ is removed from $Q$ and finds the distance in time $\mathcal{O}(\Card{E}+\Card{V}\cdot\log(\Card{V}))$~\cite{dijkstra_note_1959}.

The \emph{$A^*$} algorithm  is an extension of Dijkstra's algorithm that uses a \emph{heuristic function} $h: V \rightarrow \mathbb{N}$ to speed up the search. Assuming $h(x)$ is an estimate of the distance between $u$ and $x$, in each iteration of $A^*$, instead of removing vertex~$w$ from~$Q$, where $l(w)$ is minimal as in Dijkstra's algorithm, 
the $A^*$ algorithm removes~$w$, where $l(w) + h(w)$ is minimal. 
%
%\longversion{
We~call~$h$ % function $h$
\begin{itemize}
\item \emph{admissible}, if $h(w) \leq d_N(w, v)$ for all $w \in V$, and
\item \emph{consistent}, if $h(w) \leq \sigma(\{w,y\}) + h(y)$ for all $\{w,y\} \in E$.
\end{itemize}%}
Intuitively, admissibility of~$h$ implies that~$h$ does not ``over-approximate'',
i.e., it provides a lower bound,
and consistency of~$h$ additionally establishes a form of triangle inequality for~$h$.
While $A^*$ is correct if $h$ is admissible, polynomial runtime can only be guaranteed
if the heuristic is consistent~\cite{HartNilssonRaphael68a}.

\section{Solving the Steiner Tree Problem}
In this section, we describe our advancement to the Dijkstra-Steiner (DS) algorithm. DS combines ideas from the A* and Dreyfus-Wagner algorithms. We first discuss the Dreyfus-Wagner algorithm, followed by DS. Finally, we lift DS to more general heuristic functions, which provide a lower bound on the costs, % , %. %admissibility.
and show correctness. For this section, we assume an STP instance~$\mathcal{I}=(N,R)$ with network~$N=(V,E,\sigma)$  and set~$R$ of terminals,
where~$r\in R$ is an arbitrary terminal. The vertex~$r$ is used  
as the root of the resulting SMT. % in the following algorithms.

\subsection{The Dreyfus-Wagner (DW) algorithm}\label{chap:theory-solver-dreyfus}
The %overall idea of this 
algorithm is motivated by the fact that any SMT for instance~$\mathcal{I}$
is guaranteed to consist of so-called sub-SMTs~\cite{DreyfusWagner72a}. Given a vertex~$u \in V$ and a set $I \subseteq R$, we define a \emph{sub-SMT~$S$} for $(u, I)$ as an SMT for instance~$(N,I\cup\{u\})$,
i.e., $S\in\smt((N,I\cup\{u\}))$.
Further,
we denote by~$l^*(u,I)$ the \emph{sub-SMT costs}, $l^*(u,I) = \csmt(N, I\cup\{u\})$. The sub-SMTs of an SMT are vertex-disjoint apart from their corresponding roots.
In other words, any SMT consists of sub-SMTs
that are ``joined'' using their root vertices, which are referred to by \emph{join vertices}.
Then, intuitively, it suffices~\cite{DreyfusWagner72a} to incrementally compute 
sub-SMTs for parts of the instance and join them accordingly.
Thereby, we create larger sub-SMTs, until finally ending at an SMT for~$(N,R)$.

\begin{Example}
Recall instance~$\mathcal{I}=(N,R)$, where~$R=\{a,b,c,d\}$ of Example~\ref{ex:instance}. 
In Figure~\ref{fig:network} (right), %\ref{fig:graph-smt}, 
vertex $F$ is a join vertex %. Precisely, 
for the Steiner tree %of Figure~\ref{fig:graph-smt}, 
and vertex $F$ joins the sub-SMTs for instances $(N,\{a,b, F\})$, $(N,\{c,F\})$, and $(N,\{d,F\})$. \hfill
\end{Example}

\begin{algorithm2e}[t]
\KwData{Instance $\mathcal{I}=(N, R)$ of STP where $N=$ $(V,E,\sigma)$ and a root~$r \in R$}
\KwResult{The costs of an SMT for~$\mathcal{I}$}
$R' \gets R \setminus \{r\}$ \;
$l^*(u, \{v\}) \gets d_N(u,v)$ for all $v \in R$, $u \in V$\;\label{alg:dw-steiner:dist}
%\;
\For{$t=2$ \KwTo $|R'|$}{
	\ForEach{$I \subseteq R^\prime$ such that $|I| = t$} {
    	\lForEach{$u \in V$} {\
			$l(u,I) \gets \displaystyle\min_{\emptyset \subset J \subset I} (l^*(u, J) + l^*(u, I \setminus J))$ \label{alg:dw-steiner:comb}
         }
         \lForEach{$u \in V$} {\
         	$l^*(u, I) \gets \displaystyle\min_{v \in V}(d_N(u, v) + l(v, I))$ \label{alg:dw-steiner:prop}
         }
    }
}
\Return{$l^*(r, R^\prime)$}
\caption{Dreyfus-Wagner (DW) algorithm \cite[Ch.2]{DreyfusWagner72a}}\label{alg:dw-steiner}
\end{algorithm2e}

{Listing~\ref{alg:dw-steiner} shows DW.
%, which computes sub-SMT costs, whose corresponding sub-SMTs can be computed by retracing algorithm steps accordingly.
%The actual SMT can be 
The sub-SMT costs are computed for~$(u,I)$ with increasing cardinalities of sets~$I$. 
Sub-SMTs for~$(u,\{v\})$ (singleton of terminals) are computed using the distance between $u$ and~$v$, cf., Line~\ref{alg:dw-steiner:dist}. 
The remaining sub-SMTs are computed in two steps. 
First, in step (i), a \emph{tentative cost value~$l(u,I)$} is computed for every vertex~$u\in V$ and each set~$I$ of cardinality~$t$, by \emph{combining costs}~$l^*(u,J)$ of
sub-SMTs for~$J\subsetneq I$ accordingly, cf., Line~\ref{alg:dw-steiner:comb}.
Intuitively, this corresponds to joining two sub\hy SMTs at root~$u$.
Then, in Step~(ii), cf., Line~\ref{alg:dw-steiner:prop}, sub-SMT costs~$l^*(u,I)$ are computed for each vertex~$u\in V$, 
by \emph{propagating costs~$l(v,I)$} to all vertices~$u\in V$.
Intuitively, Line~\ref{alg:dw-steiner:prop} corresponds to connecting a vertex~$u$ to
the sub-SMT for~$(v,I)$ by a path between~$u$ and~$v$.
After the algorithm terminates, $l^*(r, R\setminus\{r\}) = \csmt(N,R)$. The SMT can then be found by retracing the steps of the algorithm.

\begin{proposition}[\hspace{-0.1pt}\cite{DreyfusWagner72a,HougardySilvanusVygen16a}]\label{prop:runtime}
Given an instance~$\mathcal{I}=(N,R)$ of STP, where~$N=(V,E,\sigma)$.
The DW algorithm runs in time~$\mathcal{O}(3^{\Card{R}})$ and space~$\mathcal{O}(2^{\Card{R}}\Card{E})$.
\end{proposition}
% \begin{proof}[Proof (Idea)]
% There are $3^{\Card{R}-1}$ disjoint subsets and $2^{\Card{R}} \cdot \Card{E}$ 
% tuples of the form $(u, I)$ for vertex~$u\in V$ and $I\subseteq 2^R$. \hfill
% \end{proof}

\subsection{The Dijkstra-Steiner (DS) algorithm}\label{chap:solver:dijstein}
In theory, the runtime of DW seems suitable for instances with a small number of terminals. 
%Unfortunately, 
In practice this algorithm consumes too much time and memory, 
even for about a few dozen of terminals. %a number of terminals that has only two digits. 
However, one can still lower the runtime by changing how the whole search space is explored.
To this end, we first define for given instance~$\mathcal{I}=(N,R)$ of STP the \emph{Steiner search network} $\mathcal{N}(\mathcal{I})\eqdef (V', E', \sigma')$, where the vertices are a set of pairs among $V\times 2^R$,~i.e., $V'\eqdef \{(u, I), u \in V, \emptyset \subsetneq I \subseteq R \setminus \{r\} \}$.
Then, there is an edge~$e=\{(u,I),(v,J)\}$ in $E'$ between 
any two distinct vertices $(u, I)$ and $(v, J)$, if we have either (1)~$\{u,v\} \in \delta_{(V,E)}(u)$ and~$I=J$; or
(2)~$u=v$ and $\emptyset \subsetneq J  \subsetneq I$.
The cost for each edge~$e\in E'$ are given as follows $\sigma'(e)\eqdef w$,
where $w=\sigma(\{u,v\})$ in Case~(1), and $w=l^*(v, I \setminus J)$ in Case~(2). 
Since~$\mathcal{N}(\mathcal{I})$ is a network, we can apply \emph{Dijkstra's algorithm} to~$\mathcal{N}(\mathcal{I})$. 
Since we cannot construct~$\mathcal{N}(\mathcal{I})$ as 
$l^*$ is not known in advance, the algorithm  runs on a partial network that is dynamically amended.
During expansion of a vertex (in Dijkstra's algorithm),
we consider only those neighboring tuples that are either adjacent in the network
according to Case~(1) or where $(u, I \setminus J)$ has been expanded before for Case~(2). 
The improved %adaption of 
%Dijkstra's 
algorithm %in our setting 
is 
called the Dijkstra-Steiner (DS) algorithm~\cite{HougardySilvanusVygen16a}. 
DS has a worst-case runtime that is similar to DW. But DS does not expand
%those 
vertices~$(u,I)$, where costs $l^*(u,I)>\csmt(N,R)$, 
which can cut down the runtime considerably in practice.

Another crucial tool to reduce the runtime is a guiding heuristic function and pruning the search space. 
A \emph{Steiner (guiding) heuristic function} $h^*: V \times 2^R \to \mathbb{N}$, provides a \emph{(cost) lower bound}\footnote{We 
provide a formal definition in the next subsection.} on $l^*(u, I)$.  Similar to the \astar algorithm, DS chooses in each iteration the tuple $(u, I)$ that minimizes $l^*(u, I) + h^*(u, R \setminus I)$. 
Thereby, the algorithm ignores tuples~$(u,I)$ with $l^*(u, I) + h^*(u, R \setminus I) > \csmt(N,R)$, which can further reduce the number of expanded tuples. 
 In addition to a heuristic function, DS uses \emph{pruning} to speed up the search. 
%Here 
An 
 \emph{upper bound on costs} is used to ignore tuples that do not contribute to an SMT.
Given such a (cost) upper bound on any SMT for $(N, I)$, where $I \subseteq R$, 
we can ignore tuples~$(u,I)$, where $l^*(u, I)$ exceeds this bound. 
Intuitively, this cuts down the number of considered tuples. %
Pruning and the Steiner heuristic function add their respective runtimes. 
While this increases the worst-case, it significantly reduces the runtime for most instances.
% This increases the worst-case runtime, but in practice significantly reduces the runtime for most instances.

%\paragraph*{Putting it together}
\begin{algorithm2e}[t]
\KwData{An STP instance $\mathcal{I}=(N,R)$, where $N=$ $(V,E,\sigma)$, root $r \in R$, and a Steiner heuristic $h^*$}
\KwResult{Pair $\langle T, c_N(T) \rangle$, where~$T$ is an SMT for~$\mathcal{I}$}
$R^\prime \eqdef R \setminus \{r\}$ \;
$l(u, I) \gets \infty$ for all $(u, I) \in V \times 2^{R^\prime}$ \;
$l(u, \{u\}) \gets 0$ for all $u \in R^\prime$ \;
$l(u, \emptyset) \gets 0$ for all $u \in V$ \;
$b(u, I) \gets \emptyset$ for all $(u, I) \in V \times 2^{R^\prime}$ \label{alg:dijkstra:bt1}\;
$Q \gets \{(u, \{u\}) \mid u \in R^\prime \}$ \;
%\;
$P \gets \emptyset$ \;
\While{$(r, R^\prime) \notin P$} {\label{alg:dijksteiner-main-loop}
	$(u, I) \gets \min_{(u, I)\in Q} l(u, I) + h^*(u, R \setminus I)$ \label{alg:dijkstra:heuristic}\; 
    $Q \gets Q \setminus \{(u, I)\}$ \;\label{alg:dijksteiner:expand}
    $P \gets P \cup \{(u, I)\}$ \;
    
       \ForEach{$\emptyset \subset J \subseteq R^\prime \setminus I$ with $(u, J) \in P$\label{alg:dijkstra1:merge}} {\label{alg:dijksteiner-set-loop}
    	\If{$l(u, I) + l(u, J) < l(u, I \cup J)$} {\label{alg:dijksteiner-if2}
        	$l(u, I \cup J) \gets l(u, I) + l(v, J)$ \;
        	$b(u, I \cup J) \gets \{(u, I), (u, J)\}$ \label{alg:dijkstra:bt3}\;
            \lIf{not $prune\_combine(u, I, J)$ \label{alg:dijkstra:prune2}} {
            	$Q \gets Q \cup \{l(u, I \cup J)\}$ \label{alg:dijksteiner-add2}
            }
        }
    }
    
    \ForEach{$\{u,v\} \in \delta_{(V,E)}(u)$}{\label{alg:dijksteiner-nb-loop}
    	\If{$l(u, I) + \sigma(\{u,v\}) < l(v, I)$} {\label{alg:dijksteiner-if1}
        	$l(v, I) \gets l(u, I) + \sigma(\{u,v\})$ \;
            $b(v, I) \gets \{(u, I)\}$ \label{alg:dijkstra:bt2}\;
            \lIf{not $prune(u, I)$ \label{alg:dijkstra:prune1}} { 
            	$Q \gets Q \cup \{(v, I)\}$\hspace{-1em} \label{alg:dijksteiner-add1}
            }
        }
    }
}
\Return{$\langle (\bigcup_{e\in E'} e, E', r), l(r, R^\prime)\rangle$, where $E'{=}\computeSMT(r, R^\prime)\hspace{-1em}$}\;
%\;
%\smallskip
%
%
\caption{The \dsstar algorithm extending  DS~\cite{HougardySilvanusVygen16a}}\label{alg:dijkstra:final}\algorithmfootnote{Given vertex~$u\in V$ and $I\in 2^R$. Then, $\computeSMT(u, I) \eqdef \bigcup_{(u^\prime, I^\prime) \in b(u,I)} \computeSMT(u^\prime, I^\prime) \cup \{\{u, u^\prime\} \mid u\neq u'\}$.}
\end{algorithm2e}

Listing~\ref{alg:dijkstra:final} presents \emph{\dsstar}, the modified variant of DS. The modifications will be the topic of the next section.
Since $l^*$ is not known in advance (as mentioned above),
\dsstar maintains \emph{tentative costs~$l(u,I)$} for every tuple~$(u,I)$,
and uses set~$P$ to track expanded tuples~$(u,I)$.
Note that, after termination of \dsstar, 
tentative costs~$l(r,R\setminus\{r\})$ 
for root vertex~$r$ are guaranteed to be the optimal 
costs of any SMT for the instance,
i.e., $l(r,R\setminus\{r\})=l^*(r,R\setminus\{r\})=\csmt(N,R)$.
Further, \dsstar keeps track of tuples~$(u,I)$ 
contributing to $l^*(r,R\setminus\{r\})$,
which is maintained in \emph{retrace set~$b(r,R\setminus\{r\})$}.
In the end, retrace sets are used for retracing a (corresponding) SMT for~$(N,R)$.
%
%everything discussed in this section is put together. 

%
%\FIX{
\dsstar as presented in Listing~\ref{alg:dijkstra:final} works similarly to DW.
In particular, the two main steps of DW, (i) combining costs (Line~\ref{alg:dw-steiner:comb}) and (ii) propagating costs (Line~\ref{alg:dw-steiner:prop}), are performed in the loops in Line~\ref{alg:dijkstra1:merge} and Line~\ref{alg:dijksteiner-nb-loop}, respectively. 
%f A$^*$ the two loops in Lines~\ref{alg:dijkstra1:merge} and \ref{alg:dijksteiner-nb-loop} correspond to expanding a vertex. 
However, the main difference between DW and \dsstar is the sequence of operations. 
Namely, in contrast to DW, \dsstar does not process sets of terminals in order of increasing cardinality, but ordered by increasing estimated costs, %Steiner heuristic function and tentative costs 
as computed in Line~\ref{alg:dijkstra:heuristic}. %, similar to A$^*$.
Unique to \dsstar is the pruning in Lines~\ref{alg:dijkstra:prune2} and~\ref{alg:dijkstra:prune1}. 
Finally, after expansion of~$(r,R')$, $\computeSMT$ is used to combine sub-SMTs in order to construct an SMT. 
For this construction, we use retracing sets %for construction 
as collected in Lines~\ref{alg:dijkstra:bt1}, \ref{alg:dijkstra:bt3}, and~\ref{alg:dijkstra:bt2}. %, are used. 
%}
%\todo{sub-solutions -- discuss monday}

\begin{Example}
%\todo{complete and adapt style}
Consider again our instance~$\mathcal{I}=(N,R)$, where~$N=(V,E,\sigma)$ from Example~\ref{ex:instance}, and let~$r\eqdef c$ be the root node.
In the following, we assume a perfect heuristic~$h^*$, where 
$h^*(u, I) = \csmt(N, I \cup \{u\})$ for any vertex~$(u,I)$ of~$\mathcal{N}(\mathcal{I})$.
Then, algorithm \dsstar requires $10$ iterations. 
The necessary operations can be tracked using the SMT in Figure~\ref{fig:network} (right), %\ref{fig:graph-smt}.
First, the initial tuples~$(a,\{a\}), (b, \{b\})$ and $(d, \{d\})$, all with costs $0$, are expanded, cf., Line~\ref{alg:dijksteiner:expand} of Listing~\ref{alg:dijkstra:final}.
Next, the tuples $(b, \{a\})$ of cost~1 and $(I, \{d\})$ of cost~5 are expanded. 
%\todo{fix and explain; what about pruning? is it used/needed here?}
Then, tuple $(b, \{a, b\})$ of cost~1 gets expanded. %, all necessary sets of terminals are present at the join vertex $F$. 
After expanding $(F, \{d\})$ of cost~10 and $(F, \{a,b\})$ cost~6, we have the tuples needed for the set~$R\setminus\{c\}$ of terminals. 
It remains to connect it to the root, which 
 is achieved by expanding $(F, \{a,b,d\})$ of cost~16 and finally $(c, \{a,b,d\})$ of cost 22 yielding $\csmt(N, \{a,b,c,d\})=22$. 
This run creates 20 tuples in $Q$.
Without the heuristic, i.e., assuming $h^*(u,I)=0$ for any vertex~$(u,I)$ of~$\mathcal{N}(\mathcal{I})$, every tuple~$(v,J)$ with costs smaller than 22 has to be expanded. 
This takes about 55 iterations and 90 entries in $Q$. 
Note that this is still considerably less than the absolute worst case of about 160 iterations as in Listing~\ref{alg:dw-steiner} (DW).\hfill
%\FIX{If we remove the data structure change, we save more than half a page}
%Listing all the steps in the calculation of our running example would take too long. Instead Table~\ref{tab:st-dijkstra-simple} shows one iteration of the algorithm, from iteration 3 to 4 without a heuristic. The first element in the queue, $b, \{a\}$ is chosen in the iteration. The entry is marked in $P$ and propagated to its neighbors. For $a$ the new sub-solution is no improvement over the known one. $F$ does not know any sub-solution for $\{a\}$ yet, so the entry is created. Furthermore, the entries for $\{a\}$ and $\{b\}$ are combined to a larger sub-solution. Finally the corresponding queue entries are created.%
%
\end{Example}
\vspace{-.9em}

\subsection{Extended Dijkstra-Steiner \dsstar for admissibility}
In this section, we discuss our changes to DS. Since $\mathcal{N}(\mathcal{I})$ is a dynamic graph structure and DS contains adaptions to the \astar algorithm % (and
% Dijkstra's algorithm, respectively),
correctness is not immediately
 obvious~\cite{HougardySilvanusVygen16a}. We therefore also discuss correctness of our new algorithm \dsstar.

We establish the following definition to  lift the existing correctness result of DS for any Steiner heuristic function
that provides a cost lower bound.
\begin{Definition}
Given an instance~$\mathcal{I}=(N,R)$ of~STP, where $N=(V,E,\sigma)$ and $r\in R$,
and a Steiner heuristic function~$h^*: V \times 2^R \to \mathbb{N}$.
%Further, let~$u\in V$ and $I\subseteq 2^R$.
%Then, we define~$l^*(u, I)\eqdef \csmt(N, I \cup \{u\})$.
Then, we say~$h^*$ is
\begin{itemize}
\item \emph{admissible}, if 
$h^*(u, I) \leq l^*(u, I)$
%l^*(u, I) + h^*(u, R \setminus I) \leq l^*(u, R)
for every $u \in V$ and $\{r\} \subseteq I \subseteq R
$; and %$h^*$ is

\item %$h^*$ is
\emph{consistent}, cf.~\cite{HougardySilvanusVygen16a}, if
$ h^*(u, I) \leq h^*(v, I^\prime) + l^*(v, (I \setminus I^\prime) \cup
\{u\})$ for every $u, v \in V$ and  $\{r\}
\subseteq I^\prime \subseteq I \subseteq R $.
\end{itemize}
\end{Definition}

\noindent Intuitively, with this definition we lift the concept of admissibility of heuristic functions~$h$ to Steiner heuristic functions~$h^*$
and establish the perspective of admissibility and consistency in the context of Steiner heuristic functions.
Similar to heuristic functions for~$A^*$, an admissible Steiner heuristic function~$h^*$ does not ``over-approximate'', i.e., it provides a lower bound,
which is a generalization of the stricter consistency notion~\cite{HougardySilvanusVygen16a} that was used before.
Note, %however, 
that a vertex of the Steiner search network 
intuitively refers to several sub-graphs of potential Steiner trees
and not to a simple path as used in~$A^*$. 
As a result, a Steiner heuristic function goes beyond plain heuristic functions 
on top of the Steiner network.
Next, we show that consistency still implies admissibility.

\begin{proposition}
Given an instance~$(N,R)$ of STP, where~$r\in R$ and a consistent Steiner heuristic function~$h^*$. Then, $h^*$ is admissible.
\end{proposition}
\begin{proof}[Proof (Sketch)]
%\FIX{
We show that~$h^*(u, I) \leq l^*(u, I)$ for every $u \in V, \{r\} \subseteq I \subseteq R$, by induction on the cardinality of $I$. 
The base case of $\Card{I} = \Card{\{r\}}=1$ is similar to finding the shortest path in~$N$ from~$u$ to~$r$, as used in A$^*$.
In particular, in this case consistency and admissibility of Steiner heuristic functions are
%In this case, consistency and admissibility of Steiner heuristic functions are 
a special case of consistency and admissibility of heuristic functions, respectively. 
Therefore, the result follows by the same argument~\cite{dechter_1985}, as in the original work for~A$^*$.
In the induction step we show that~$h^*(u, I) \leq l^*(u, I)$ for every $u \in V, \{r\} \subseteq I \subseteq R$, assuming that~$h^*(v, J) \leq l^*(v, J)$ for every $v \in V, \{r\} \subseteq J \subsetneq I$.
Thereby, we use a result~\cite{DreyfusWagner72a}, which shows that for each 
vertex $u\in V$, either one of the cases holds:
\begin{enumerate}
	\item $l^*(u, I)$ = $l^*(u, J) + l^*(u, I \setminus J)$, with $J \subsetneq I$:
	By consistency of~$h^*$ and by the induction hypothesis, we have $h^*(u, I) \leq h^*(u, J) + l^*(u, (I \setminus J) \cup \{u\}) \leq l^*(u, J) + l^*(u, I \setminus J) \leq l^*(u, I)$.
	\item $l^*(u, I)$ = $l^*(v, I) + d_N(u, v)$, for $u \neq v$, and for $v$ we have Case~1. Then, we conclude by consistency of~$h^*$ and by Case~1: $h^*(u, I) \leq h^*(v, I) + l^*(v, (I \setminus I) \cup \{u\}) \leq l^*(v, I) + d_N(u, v) = l^*(u, I)$.%\\[-2.5em]%
        %\hfill
        \end{enumerate}%
%}%
\vspace{-1.5em}
\end{proof}
%and consistency of~$h^*$ additionally establishes a form of triangle inequality for~$h^*$.
%
%Note that consistency implies admissiblity. 

\vspace{-.5em}
\noindent However, consistency is indeed strictly stronger, i.e., one can easily construct examples of non-consistent, admissible Steiner heuristic functions. % that are not consistent.

\begin{Example}
  Consider network~$N$ given in Figure~\ref{fig:network} (left). %\ref{fig:graph-base}. 
Let $r\eqdef c$ and $h^*$ be an admissible Steiner heuristic function. Take $h^*(d, \{c\}) \eqdef 16$ and $h^*(F, \{c\})\eqdef 1$. Observe that $h^*$ is admissible, but not consistent, as $h^*(d,\{d\}) > h^*(F,\{c\}) + l^*(F,\{d\})$, i.e., $16 > 1 + 10$.
  
  \hfill
\end{Example}

\vspace{-1em}
\noindent So far, correctness of DS is known for consistent Steiner heuristic functions~\cite{HougardySilvanusVygen16a}.
In this case, consistency guarantees optimality of expanded tuples, i.e., 
for every expanded tuple~$(u,I)$ (in $P$, cf., Listing~\ref{alg:dijkstra:final}), we have~$l(u,I)=l^*(u,I)$.
%On the contrary, 
%in $P$ of Listing~\ref{} is optimal. 
On the contrary, in case of mere {admissibility},
which, intuitively, ensures only a cost lower bound,
the algorithm might find a lower value for~$l(u,I)$ even after expansion of~$(u,I)$, i.e. it only holds that $l(u,I) \geq l^*(u,I)$. 
Recall Listing~\ref{alg:dijkstra:final},
which indeed presents the \dsstar algorithm.
Actually, only Lines~\ref{alg:dijksteiner-if1} and~\ref{alg:dijksteiner-if2} of Listing~\ref{alg:dijkstra:final} differ slightly from the original DS algorithm~\cite{HougardySilvanusVygen16a}. During expansion, \dsstar explicitly considers already expanded tuples in $P$ since the tentative cost value~$l(v,I)$ for an expanded 
tuple~$(v,I)\in P$ is not guaranteed to be optimal.

Although the changes are small, the effect is significant: new Steiner heuristic functions can be used with \dsstar and since the proof of correctness for~DS strongly depends on tentative costs of expanded tuples being optimal, we have to use a different approach for showing correctness of~\dsstar. \footnoteitext{\label{lab:fullproofs}Proofs of theorems marked with~``$\star$'' can be found in  the appendix.}%
\begin{Theorem}[$\star^{\ref{lab:fullproofs}}$, main result]\label{thm:main}
For an instance $(N,R)$ of STP where~$N=(V,E,\sigma)$ and a root~$r\in R$,
 the \dsstar Algorithm given in Listing~\ref{alg:dijkstra:final}, terminates if $h^*$ is admissible. 
Further, after termination, $l(r, R \setminus\{r\})$ is the cost of an SMT for the instance~$(N, R)$, i.e., $l(r, R \setminus\{r\})=l^*(r, R\setminus\{r\}=\csmt(N,R)$. 
\end{Theorem}
\begin{proof}[Proof (Idea)]
Since~$h^*$ is admissible, we have $h^*(r,R \setminus\{r\})=0$.
By construction of~\dsstar and admissibility of~$h^*$, 
every tuple~$(u,I)$ such that $l(u,I)+h^*(u,R\setminus I)<l(r,R\setminus \{r\})$ is expanded.
As a result, after termination, every tuple~$(u,I)$ that ``contributes'' to~$\csmt(N,R)$,
is expanded beforehand.\hfill
%In every iteration, we consider both, tuples in~$Q$ and tuples in~$P$. 
%%sub-solutions to eventually construct an SMT in $Q$ and/or $P$. 
%Further, $Q$ always contains the necessary entries, such that the construction of a SMT proceeds. 
%As long as $h$ is admissible, the construction will eventually continue. 
%%The full proof can be found .
\end{proof}

\noindent While for consistent Steiner heuristic functions~$h^*$, one obtains for~DS 
a similar runtime result as for~DW, cf., Proposition~\ref{prop:runtime},
this is not guaranteed for admissible functions~$h^*$ and~\dsstar.
Intuitively, as described above, for non-consistent Steiner 
heuristic functions~$h^*$,
%consistent,
% Steiner heuristic function implies, that, given a vertex $v$, 
paths to arbitrary tuples~$(v,I)$ are not necessarily cost optimal.

\begin{Theorem}
Given an instance $\mathcal{I}=(N, R)$ of STP, where $N = (V, E, \sigma)$, any root~$r\in R$, and an admissible Steiner heuristic function~$h^*$.
Then, the \dsstar algorithm runs in time $\mathcal{O}(|R| \cdot (|V| + 2^{|R|})^{|V| + |R|})$.
\end{Theorem}
\begin{proof}[Proof (Sketch)]
Each vertex of the form $(u, I)$ in the Steiner network~$\mathcal{N}(\mathcal{I})$ has by construction at most $|V| + 2^{|R| - |I|}$ neighboring 
tuples in the network. 
Observe that an SMT of~$\mathcal{I}$ consists of at most $|V|$ many vertices and 
has at most $|R|$ join vertices.
Then, the length of a path from any tuple $(u, I)$ to $(r, R\setminus \{r\})$ in~$\mathcal{N}$ that is discovered in \dsstar is at most $|V| + |R|$,
since a longer path would exceed $\csmt(N,R)$, causing algorithm \dsstar to not expand the corresponding tuple. 
We therefore reach the goal $(r, R\setminus\{r\})$ after expanding at most $(|V|+2^{|R|-1})^{|V|+|R|}$ %many %sorry space first
vertices. 
In the worst case this is done for each terminal, except the root~$r$. Hence, the bound holds.\hfill
\end{proof}
%olving algorithm's runtime does not change, if $h$ is consistent and runs sub-exponentially, we have no bound in case of mere admissibility. 

\noindent In theory, the worst-case runtime may worsen 
significantly in case of mere admissibility, similar to the situation of exponential blow up for the \astar algorithm~\cite{zhang_astar_2009}.
However, in practice the ability to use non-consistent, but admissible heuristics often provides tighter bounds leading to shorter runtimes. 
The same is true for the \astar algorithm, where at first nobody could imagine merely admissible heuristic functions to outperform consistent ones~\cite{martelli_complexity_1977, zhang_astar_2009}. For \dsstar as well as \astar it turns out that ILP based heuristic functions are both not consistent and often well-suited. In the next section, we discuss the Steiner heuristic functions used in our solver \dsstarsolver.

\section{Implementation of \dsstar: \emph{\dsstarsolver}} \label{chap:solvingMain}
We created the solver \emph{\dsstarsolver}, a fully functional STP solver written in C++\footnote{\dsstarsolver is publicly available at \url{https://github.com/ASchidler/steiner_cpp}~\cite{SchidlerHecherFichte20}.}. 
\dsstarsolver consists of two modules: the \emph{solving module}, using \dsstar as presented in Listing~\ref{alg:dijkstra:final}, and the \emph{preprocessing module}, trying to reduce the complexity of the instance, before solving is started.
In the following, we discuss implementation details regarding both these modules.
Within this section, we again assume a given instance~$(N,R)$, where~$N=(V,E,\sigma)$ and a root~$r\in V$.

\subsection{Solving Module}
\dsstar is implemented as previously described with the same pruning technique as originally suggested for DS~\cite{HougardySilvanusVygen16a}. The main improvement is the choice of Steiner heuristic functions, which is crucial for good results, as our experimental results show. We use two different Steiner heuristic functions and will discuss them subsequently: \emph{dual ascent} and \emph{1-tree}.
Given a lower bound~$lb(N,I)$ of~$\csmt(N,I)$ for any~$I \subseteq R$,
we define the \emph{corresponding Steiner heuristic function~$h^*_{lb}$ (using~$lb$)} by~$h^*_{lb}(u,I)\eqdef lb(N,I\cup\{u\})$ for any~$u\in V$, $I\subseteq R$.

\medskip\noindent\textbf{Dual Ascent}\cite[Section~3]{wong_dual_1984}\textbf{.} This admissible, but not consistent, Steiner heuristic function can only be used with \dsstar and not DS. The idea is to use a feasible solution for the dual of the ILP formulation as a lower bound. The algorithm starts with edge costs $\sigma'=\sigma$ and $|R \setminus \{r\}|$ sub-graphs of $(V,E)$, where each sub-graph $C_t$ consists of exactly one non-root terminal $t \in R \setminus \{r\}$. 
In each iteration one sub-graph $C_t$ is selected and extended. 
First, the set $E_t$ is computed: $E_t$ consists of all edges $e \in E$ incident to a node in $C_t$, but not in $C_t$. 
Next, the lowest cost $c' = \min_{e\in E_t} \sigma'(e)$ is determined. 
Then, all edges $e \in E_t$ with $\sigma'(e)=c'$ are added to $C_t$ and the costs of all edges in $E_t$ are reduced by $c'$ in $\sigma'$. 
Eventually, the sub-graphs will form one connected sub-graph $C_r$, containing all terminals including the root. Therefore, $C_r$ is a Steiner tree for the instance. The sum of all $c'$ when adding an edge is a lower bound for $\csmt(N,R)$. Whenever we select a sub-graph $C_t$, we choose the sub-graph that has the minimal number of incident edges~\cite{polzin_tobias_algorithms_2003,PajorUchoaWerneck18a}.
The algorithm runs in $\mathcal{O}(|E| \cdot \min\{\Card{V} \cdot \Card{R}, |E|\})$ \cite{duin_steiners_1993}.

%%%%%%%%%%%%%%%%%%%%%%%%%%%%%%%%%%%%%%%%%%%%%%%%%%%%%%%%%%%%%%%%%%%%%%%%%%%%%%%%%%%%%%%%%%%%%%%%%%%%%%%%%%%%%%%%%%
%\paragraph*{\glspl{mst} and 1-Tree Steiner heuristic functions}\label{chap:mst-bound}
\medskip\noindent\textbf{1-Tree.} Alternatively, \dsstarsolver uses the \emph{1-Tree} lower bound, whose corresponding Steiner heuristic function is \emph{consistent}. This method is used in the original DS implementation~\cite{HougardySilvanusVygen16a}. 

\begin{proposition}[1-tree lower bound]\cite[Lemma~8]{HougardySilvanusVygen16a}\label{thm:1tree}
Given an STP instance $(N, R)$. Let $r \in R$ be any terminal and $R^\prime = R \setminus \{r\}$. Further, let $c$ be the cost of any minimum spanning tree of $D_N(R^\prime)$. Then,
$
	\frac{1}{2}( c + \min_{u, v \in R^\prime: u\neq v\text{ or }|R^\prime| = 1} %\frac{d_N(r, u) + d_N(r, v)}{2}
(d_N(r, u) + d_N(r, v))
$
is a cost lower bound of any SMT of~$(N,R)$. 
\end{proposition}
The distance network can be constructed in time $\mathcal{O}(\Card{V} \log(\Card{V}) + \Card{E})$ and the minimum spanning tree for this network in time $\mathcal{O}(\Card{R}^2)$. The lower bound can then be obtained in time $\mathcal{O}(\Card{R})$ \cite{HougardySilvanusVygen16a}.

\dsstarsolver only uses one of the corresponding Steiner tree heuristic functions per instance. Oftentimes, dual ascent computes tighter bounds and \dsstar therefore requires fewer iterations to succeed. 
Unfortunately, the runtime for dual ascent also increases faster with increasing graph size than the runtime of 1-tree. Hence, we use dual ascent for graphs of up to 10.000 edges and use the 1-tree heuristic function for larger graphs.

The root terminal is not chosen arbitrarily. \dsstarsolver tries different terminals as the root with dual ascent. The terminal giving the highest lower bound is then used as the root. 
This strategy delivered superior results, compared to selecting the last terminal as used~originally~\cite{HougardySilvanusVygen16a}.

%\vspace{-2em}
\subsection{Preprocessing Module}\label{chap:preprocessing}
Preprocessing for STP identifies and removes vertices and edges that are not required for an SMT. These methods usually run in time $\mathcal{O}(\Card{V} \log(\Card{V}))$ and can quickly %and significantly
 reduce the complexity of an instance. Many instances would be too hard for the solving module alone, without preprocessing.

We use the standard methods described in literature~\cite{duin_reduction_1989, HwangRichardsWinter92a}. Preprocessing is applied until no more reductions are possible. The order of preprocessing operations is chosen in a way to maximize reuse of calculated information, e.g. distances in the network. One noteworthy preprocessing method is based on the previously discussed dual ascent lower bound.

Dual ascent calculates a lower bound for $\csmt(N,R)$ and can be extended to provide a lower bound on the costs of any SMT containing a specific vertex or edge. Given an upper bound for $\csmt(N,R)$ we can remove any vertices and edges, where this lower bound exceeds the upper bound. We use different methods to obtain good upper bounds and thereby maximize the number of removals.

Upper bounds are obtained using the \emph{repeated shortest path heuristic (RSPH)}. RSPH is a well-established STP heuristic~\cite{takahashi_approximate_1980, de_aragao_implementation_2002}, that can compute Steiner trees in time $\mathcal{O}(|R|\cdot\Card{V}^2)$. 
We can often find a tighter upper bound by running RSPH using only a sub-graph that contains all terminals. In addition to using the full instance, we run RSPH also using the following sub-graphs:
\begin{enumerate}
\item The sub-graph $C_r$ after a dual ascent run~\cite{polzin_tobias_algorithms_2003}.
\item A preprocessed instance after guessing an upper bound~ \cite{polzin_tobias_algorithms_2003}.
\item A sub-graph obtained from combining several Steiner trees~\cite{ribeiro_hybrid_2000, polzin_tobias_algorithms_2003}.
\end{enumerate}

Given a Steiner tree, we can also find a Steiner tree of lower costs by applying local search. This is done by systematically replacing sub-graphs of the tree in an effort to find a tree of smaller size. These sub-graphs may be single vertices or whole paths \cite{uchoa_fast_2012}.

\newcommand{\set}[1]{\texttt{#1}}
\newcommand{\solver}[1]{\textit{#1}}
\section{Experimental Work}% Implementation details and results}
\label{chap:benchmarks}
We conducted a series of experiments using standard benchmark
instances for STP. Instances and results\footnote{See:~\cite{FichteHecherSchidler20}.}, including raw data, are publicly available.
Our experimental work aims for a comparison between
different heuristics  in order to understand 
whether our proposed extensions  are valuable in practice. Further,
we compare the effectiveness of our prototypical solver with
other established implementations.

\medskip
\noindent\textbf{Benchmark Instances.}
We considered a selection of overall 1.623 instances, which originate
from the 11th DIMACS Challenge~\cite{JohnsonEtAl14a} and PACE~2018
Challenge~\cite{BonnetSikora19a}. We group them as done in the
literature~\cite{IwataShigemura19a}.
\begin{enumerate}
\item\set{SteinLib},\\
  which contain generated graphs with random costs (\set{random}),
  artificial instances (\set{artificial}), instances with euclidean
  weights (\set{euclidean}), cross-grid graphs (\set{crossgrid}), grid
  graphs with holes (\set{vlsi}), randomly generated rectilinear
  instances (\set{rectilinear}), wire routing problem based instances
  (\set{group});
%~\cite{KochMartinVoss2000}
% \begin{enumerate}
% \item\set{random} (generated graphs with random costs);
% \item\set{artificial} (artificial instances);
% \item\set{euclidean} (instances with euclidean weights);
% \item\set{crossgrid} (cross-grid graphs);
% \item\set{vlsi} (grid graphs with holes);
% \item\set{rectilinear} (randomly generated rectilinear instances);
% \item\set{group} (wire routing problem based instances);
% \end{enumerate}
% \item\set{vienna} Real world instances from telecommunication networks \cite{LeitnerLjubicLuipersbeckProsseggerResch2014}.
\item\set{Cph14}\\ (simplified obstacle-avoiding rectilinear instances);
\item\set{Vienna}\\ (real world instances from telecommunication networks);
\item\set{PACE2018} (PACE~2018 challenge instances).
\end{enumerate}

\medskip
\smallskip\noindent\textbf{Measure, Setup, and Resource Enforcement.}
Our results were gathered on a cluster running Ubuntu 18.04.3 LTS
(kernel 4.15.0-101-generic) and GCC 7.5.0.  Each node
is equipped with two Intel Xeon E5-2640v4 CPUs and 160GB RAM. 
We limited the solvers to 1800 seconds wall clock time and 8GB of RAM per instance. We used \emph{reprobench}\footnote{\url{https://github.com/rkkautsar/reprobench}} to setup the benchmarks and to enforce the resource limits.

\begin{table}
\centering
% \begin {tabular}{l|r||r|r@{\hskip 0.050in}Hr@{\hskip 0.050in}r@{\hskip 0.050in}r@{\hskip 0.050in}r@{\hskip 0.050in}} 
\begin {tabular}{lHrrrr} 
  \toprule
  %&	&	 & \multicolumn{1}{c|}{\dsstarsolver} & \multicolumn{2}{c}{HSV}   \\ 
  Set & Subset & N & \dsstar & \dsstar-da & \solver{HSV}\\
  \midrule
  \set{SteinLib}    & 	& 820 	& \bf{691} 	& 522 	& 557	\\
  % \set{SteinLib}	& \set{artificial}	&21	&\bf{11}	&\bf{11}	&7\\
  %       & \set{crossgrid}	&45	&\bf{35}	&28	&31\\
  %       & \set{euclidian}	&26	&\bf{26}	&\bf{26}	&\bf{26}\\
  %       & \set{group}		&89	&\bf{81}	&13	&14\\
  %       & \set{random}		&418	&\bf{319}	&227	&258\\
  %       & \set{rectlinear}	&93	&\bf{93}	&91	&\bf{93}\\
  %       & \set{vlsi}		&128	&126	&126	&\bf{128}\\	
\set{Cph14}				&		&10	&\bf{10}	&\bf{10}	&\bf{10}\\
\set{Vienna} 			&		&6	&\bf{4}	&3	&3\\
\set{PACE2018}			&		&195				&\bf{185}		&	157		& 157 \\

  \midrule
  $\Sigma$		&		& 1031	&\bf{890}	&692	&727\\
  \bottomrule
\end{tabular} 
\caption[The benchmark results for Steinlib instances.]{%
  % \footnotesize
  Number of solved instances with less than 64 terminals for DS-based solvers.
  % \dsstar and \solver{HSV}. 
  \dsstar-da lists the results for \dsstar without dual ascent  heuristic.
  \emph{N} indicates the total number of instances.
  %\emph{S} refers to the number of solved instances and %
  %\emph{U} refers to the number of uniquely solved instances. 
}
%\todo{add multicolumn for \dsstarsolver, add column for DSA}
\label{tab:bench-steinlib-short}
\end{table}
\medskip \smallskip\noindent\textbf{Benchmarked Solvers.}  We tested
three configurations of our solver: \dsstar, which implements our
enhanced algorithm without preprocessing, \dsstarsolver, which in
addition includes preprocessing, and \emph{-da}, which on uses the
1-tree heuristic function.
%
%Further, %
We test the solver \emph{HSV}\footnote{We thank the authors for
  providing us with a copy of the solver.}, which is known as a
successful implementation of DS
algorithm~\cite{HougardySilvanusVygen16a}.
We include the best PACE~2018 solvers
\emph{Pruned}\footnote{\url{https://github.com/wata-orz/steiner_tree}}
(2700bdc/Rust1.36.0)~\cite{IwataShigemura19a} and
SCIP-Jack\footnote{\url{https://scip.zib.de/}} (6.0.2/SoPlex4.0.2). 
\medskip
\noindent\textbf{Experiment 1.}
In order to benchmark the effectiveness of our enhanced algorithm and
a comparison to plain DS, we take \dsstar and the solver \solver{HSV}
into account. Since one might argue that also implementation specific
tricks and algorithm engineering might have a strong influence on the
number of solved instances, we also tested our solver in the
configuration \emph{-da}, which disables the dual ascent heuristic and
assembles a solver that is close to the underlying algorithm of
\solver{HSV}.
Since \solver{HSV} can handle only instances of less than 64
terminals, we restrict the instances accordingly.

\noindent\emph{Result:} Table~\ref{tab:bench-steinlib-short} the results on the
number of solved instances.
\dsstar without dual ascent performs worse than HSV.  \dsstar performs
better on all sets. If we subgroup \set{SteinerLib}, one can observe
that it particularly helps to solve randomly generated graphs
(\set{random}) and wire routing problem based instances (\set{group}).
%
% Helps particularly to solve instances from the groups ``group'' and random
%
%
%
%  most instance set, except vlsi, and often by a significant margin (on the sets~\set{group} and \set{random} the difference is
% large,~i.e.,~67 and 61, respectively).

\noindent\emph{Discussion:}
Besides implementation details, \dsstar-da performs worse than
\solver{HSV}. We suspect that the main reason is the use of the
Steiner heuristic functions in \solver{HSV}, which is slightly more
sophisticated than the 1-tree heuristic
function~\cite{HougardySilvanusVygen16a}.
%
% Given these differences, the gains obtained by activating dual ascent are even more striking.
%
%
The results show that using merely admissible Steiner heuristic functions can significantly improve the number of solved instances.
%
% In Table~\ref{tab:bench-steinlib-short} we
% easily observe a great increase in the number of solved instances when using dual ascent, even without preprocessing. 

\medskip\noindent\textbf{Experiment 2.}
In order to contrast our algorithm and implementation with approaches
and their implementations, we take various state-of-the-art solvers
%latest state-of-the-art solvers 
in the field into account.

\noindent\emph{Results:} In Table~\ref{tab:bench-config}, we report the
number of solved instances for the implemented techniques.
Depending on the instance set, we observe a varying number of solved
instances.
Overall, \solver{SCIP} solves the most instances. \dsstarsolver solves
237 instances less and \solver{Pruned} another 39 instances less.
If we subgroup \set{SteinLib}, \dsstarsolver and \solver{Pruned} solve
more instances than \solver{SCIP} on 
\set{vlsi}. \solver{Pruned} solves more instances than \solver{SCIP}
on \set{group}.
\solver{SCIP} solves almost 40 and 80 instances more on the sets
\set{rectlinear} and \set{random}, respectively, and similar many or a
few more on the remaining sets.
Figure~\ref{fig:scatter} shows a runtime comparison between
\dsstarsolver and SCIP based on the instance's number of terminals. We
observe that for lower numbers \dsstarsolver is usually faster than
SCIP.

\begin{table}
\footnotesize
\centering
% \begin {tabular}{
% @{\hspace{0.15em}}l@{\hspace{0.25em}}l@{\hspace{0.15em}}|@{\hspace{0.15em}}r@{\hspace{0.15em}}||
% @{\hspace{0.15em}}r@{\hspace{0.15em}}|
% @{\hspace{0.15em}}r@{\hspace{0.15em}}|
% @{\hspace{0.15em}}r@{\hspace{0.15em}}|
% @{\hspace{0.15em}}r@{\hspace{0.15em}}
% }
%DS$^*$\hspace{-0.25em}
\begin {tabular}{lHrrrrr}
  \toprule
   Set & Subset 		& N 	& \dsstarsolver 	& \dsstarsolver{}-da 	& \solver{Pruned} 	& \solver{SCIP} \\
  \midrule
  \set{SteinLib}
   & 	& 1186 	& 928 	& 796 	& 874 	& \bf{1055}	 \\
%   \set{SteinLib}
%     & \set{artificial}	& 58 	& 11 	& 11 	& 11 	& \bf{13}	 \\
%   & \set{crossgrid}		& 54	& 39	&32& 43	& \bf{54}	\\
% &  \set{euclidean}		& 29	& \bf{29}	& \bf{29}	& \bf{29}	& \bf{29}	\\
%   & \set{group}			& 126	& 104	&38& \bf{116}	& 104	\\
%     & \set{random}		& 509	& {395}	&368& 305	& \bf{473}	\\
%       & \set{rectlinear}	& 257	& 207	&181& 223	& \bf{245}	\\
%  &  \set{vlsi}			& 153	& {143}	&137& \bf{146}	& 116	\\
%\midrule
  \set{Cph14}	&& 21	& 15	&12& 16	& \bf{17}	\\
  \set{Vienna}		&& 216	& 11	&5& 26	& \bf{120}	\\
   \set{PACE2018}	&		& 200	& \bf{{189}}	&178& \bf{189}	& 178	\\
  \midrule
   $\Sigma$ &  &1623&1143&991&1105&\bf{1370}\\
  \bottomrule
\end {tabular} 
\caption{%
  Number of solved instances for the considered solvers. %
  $N$ indicates the overall number of instances. %  
  %
  % \dsstarsolver{}$-$da indicates disabled dual ascent.
  %
}
\label{tab:bench-config}
\end{table} 
\begin{figure}[t]
  \centering
    \includegraphics[scale=0.90]{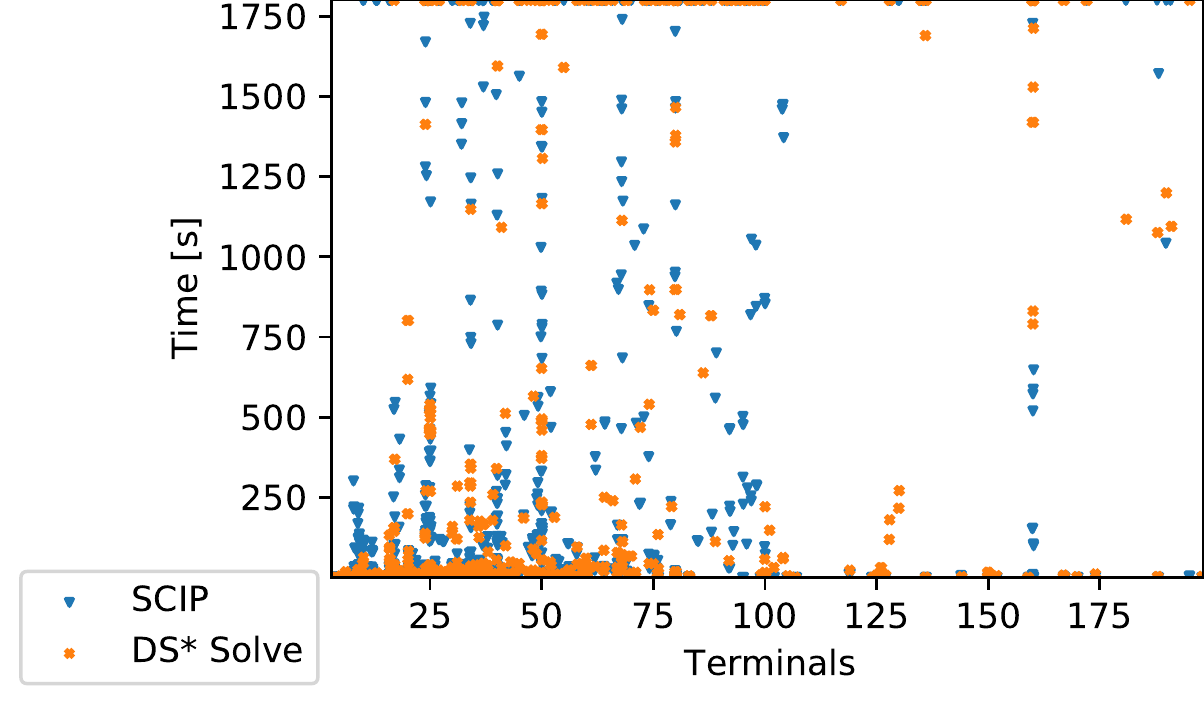}%
  \caption{Runtime of \dsstarsolver vs \solver{SCIP}. The x-axis labels the number of terminals of the instance and y-axis captures the runtime in seconds.}
  \label{fig:scatter}
\end{figure}

\noindent\emph{Discussion:}
Similar as in the previous experiment, we can see that there is a big
difference between using consistency and admissibility, namely 152
solved instances.
If we compare \dsstarsolver to \solver{Pruned}, we observe that
overall the former solves more instances. Since \solver{Pruned}
performs better on some groups of the set \set{SteinLib} sets, a
combination of both approaches might produce a better overall solver.
% We do not need the reviewers to put special focus here
%
% In particular, as \dsstarsolver and \solver{Pruned} implement
% descendants of the Dreyfus-Wagner algorithm.
%
%
%
While \solver{SCIP} solves the most instances mainly due to the high
number of solved instances in the sets \set{rectlinear} and
\set{random}, \dsstarsolver solves instances with a low number of
terminals very fast and often much faster than \solver{SCIP}.
%
% works better on arbitrary instances. 
% %
% However, when working with instances of lower terminals, \dsstarsolver can likely solve the instance faster, as Figure~\ref{fig:scatter} shows. 
%
Hence, we expect good performance as part of a portfolio solver.

\section{Conclusion and Future Work}
In this paper, we established the concept of admissible Steiner heuristic functions for the Steiner tree problem.
We focused on instances with few terminals,
and lifted the so-called Dijkstra-Steiner (DS) algorithm from consistent heuristic functions to admissibility, resulting in the \dsstar algorithm.
Intuitively, admissibility of heuristic functions only requires a weak condition for lower bounds. More precisely, an admissible heuristic function is guaranteed to never over-approximate the actual costs,
and is indeed strictly weaker than consistency.
Admissible heuristic functions enable lower bound computation based on LP techniques, as for example the dual ascent method.
Our solver \dsstarsolver
 combines the usage of admissible heuristic functions during solving,
efficient preprocessing techniques, and methods for obtaining strong upper bounds.
Experiments indicate that admissible heuristic functions have a strong effect on the solving performance.

An interesting question for future work is the integration of the dual
ascent technique into other approaches and solvers.

% related solvers to analyze admissible heuristic functions in cooperation with other techniques.
% We expect that with an improved strong dual ascent technique, we can speed up  other solvers as well.

% For future research, we are interested in incorporating our findings into other related solvers to analyze admissible heuristic functions in cooperation with other techniques.
% We expect that with an improved strong dual ascent technique, we can speed up  other solvers as well.
% %With our established notion of admissibility, 

\clearpage
\bibliographystyle{plain}
\bibliography{steiner_arxiv}
\clearpage
\begin{appendix}

\section{Additional Examples}

\subsection{MST of our example network.}

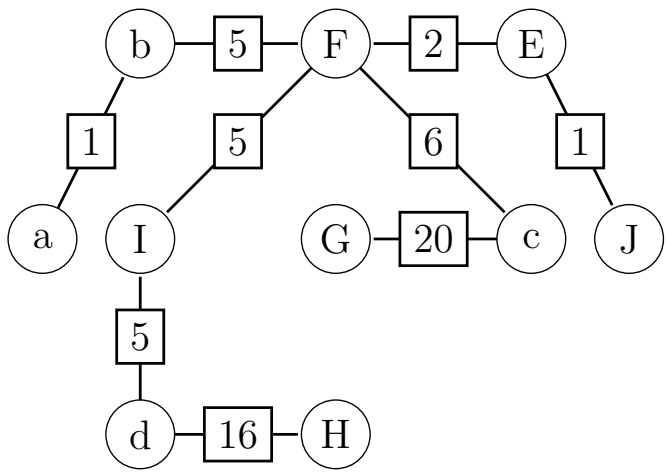
\begin{figure}[t]
  \centering
  \resizebox{.65\linewidth}{!} {%
    \begin{tikzpicture}[shorten >=1pt]
     \SetGraphUnit{2}% modifie la distance entre les nodes
     \GraphInit[vstyle=normal]
     \tikzset{LabelStyle/.style =   {draw,
                                     %fill  = yellow,
                                     %text  = red
                                     }}
 	\tikzset{VertexStyle/.style = {shape          = circle,
     							draw,
                                 minimum size   = 20 pt,
                                 inner sep      = 2pt,
                                 outer sep      = 0pt
                                 }}
     \Vertex{F}
     \SOWE(F){I} 
     \EA(I){G}
     \SO(I){d}
     \EA(d){H}
     \EA(G){c}
    \WE(F){b}
     \EA(F){E}
    {
    \SetGraphUnit{1}
     \WE(I){a}
    \EA(c){J}
    }

     %\tikzset{EdgeStyle/.append style = {bend left}}
     \Edge[label=1](a)(b)
     \Edge[label=2](E)(F)
     \Edge[label=5](b)(F)
     \Edge[label=6](F)(c)
     \Edge[label=20](c)(G)
     \Edge[label=16](d)(H)
     \Edge[label=5](d)(I)
     \Edge[label=5](F)(I)
     \Edge[label=1](E)(J)
\end{tikzpicture}%
  }
  \caption{An MST of network~$N$ in Figure~\protect\ref{fig:network}}%\ref{fig:graph-base}}
  \label{fig:graph-mst}
\end{figure}

\begin{Example}
Recall instance~$\mathcal{I}=(N,R)$ from Example~\ref{ex:instance}.
Further, Figure~\ref{fig:graph-mst} shows the MST for network~$N$ of Figure~\ref{fig:network}.%\ref{fig:graph-base}.
\end{Example}

\subsection{Example of the DW algorithm.}\label{sec:exdwa}
\begin{table*}[t]
\centering
\begin{tabular}{|l||r|r|r||rr|rr|rr||rr|}
\hline
& $\{a\}$	& $\{b\}$	& $\{d\}$	& 
\multicolumn{2}{c|}{$\{a,b\}$} &
\multicolumn{2}{c|}{$\{a,d\}$} &
\multicolumn{2}{c|}{$\{b,d\}$}& 
\multicolumn{2}{c|}{$\{a,b,d\}$}\\
\hline
Vertex & $l^*$ & $l^*$ & $l^*$ & $l$ & $l^*$ & $l$ & $l^*$ & $l$ & $l^*$ & $l$ & $l^*$\\
\hline
a & 0	& 1		& 16	& 1	 &1		& 16 &16	& 17 &16	& 16& \\
b & 1	& 0		& 15	& 1  &1		& 16 &16	& 15 &15	& 16& \\
c & 12	& 11	& 16	& 23 &17	& 28 &22	& 27 &26	& 33&\emph{22} \\
d & 16	& 15	& 0		& 31 &16	& 16 &16	& 15 &15	& 31& \\
E & 8	& 7		& 12	& 15 &8	    & 20 &18	& 19 &17	& 20&\\
F & 6	& 5		& 10	& 11 &6	    & 16 &16	& 15 &15	& 16&\\
G & 32	& 31	& 36	& 63 &32	& 68 &42	& 67 &41	& 68&\\
H & 32	& 31	& 16	& 63 &32	& 48 &42	& 47 &41	& 48&\\
I & 11	& 10	& 5		& 21 &11	& 16 &16	& 15 &15	& 16&\\
J & 9	& 8		& 13	& 17 &9	    & 22 &19	& 21 &18	& 22&\\
\hline
\end{tabular}
\caption{The full \emph{Dreyfus-Wagner} computation for our running example.}
\label{tab:dw-full}
\end{table*}

\begin{Example}

%\FIX{
Table~\ref{tab:dw-full} shows the result of a full DW run for instance $\mathcal{I}=(N,R)$, where~$N=(V,E,\sigma)$ as given in Example~\ref{ex:instance}, where $c$ is used as a root. 
The table shows vertically the vertices~$v\in V$, whereas horizontally subsets $I\subseteq 2^{R\setminus\{c\}}$ are depicted. 
Depending on the column, the cells either show value $l(v, I)$ or $l^*(v, I)$, respectively. 
The subsets are ordered by increasing cardinality.
As the computation is quite tedious, we focus on specific ones, namely on the row for the vertex $c$. 
The computation thereby proceeds column by column from left to right.
%}

For subsets of cardinality one, values of~$l^*(c,\{a\})$, $l^*(c,\{b\})$, $l^*(c,\{d\})$ are simply the distances between $c$ and $a$, $b$ and $d$, respectively.
Therefore, the corresponding values in the table are $12$, $11$ and $16$, respectively. 
For the value of $l(c, \{a,d\})$ we have to consider the values of all disjoint subsets that comprise $\{a, d\}$. 
There is only one such decomposition, $(\{a\}, \{d\})$, and the value of $l(c,\{a,d\})$ is therefore $l^*(c, \{a\}) + l^*(c, \{d\}) = 12+16 = 28$.
%}

For the computing values of~$l^*$ for set $\{a, d\}$, we have to consider the values of~$l$ for all other vertices and add the distance to $c$. 
We therefore compute $l(v, \{a, d\}) + d_N(v, c)$ for all vertices $v \in V$. 
The corresponding values are $a: 16 + 12 = 28$, $b: 16+11 = 27$, $d: 16 + 16 = 32$, $E: 20 + 8 = 28$ and $F: 16 + 6 = 22$. 
The values in relation to any other vertex would be higher, so we ignore them. We keep the minimum of $22$ as value for $l^*(c,\{a,d\})$.

Next, we compute the value of $l$ for $\{a, b, d\}$. 
We therefore have to consider the subset tuples $(\{a, b\}, \{d\})$, $(\{a\}, \{b, d\})$ and $(\{a, d\}, \{b\})$. The respective values are $17 + 16 = 33$, $12 + 26 = 38$ and $22 + 11 = 33$. We again retain the minimum value $33$.

The final tuple we need, is the value of~$l^*$ for $\{a, b, d\}$. 
We again use the $l$ values of all other vertices. 
We again need only vertices up to $F$ as $l(F, \{a, b, d\}) + d_N(c, F)$ is the minimum value. 
The computations are $a: 16 + 12 = 28$, $b: 16 + 11 = 27$, $d: 31 + 16 = 47$, $E: 20 + 8 = 28$ and $F: 16 + 6 = 22$. 
We now have found our solution: $22=l^*(c,\{a,b,d\})$.
%}
%
\end{Example}

\section{Further algorithmic details}
%
%
%\section{DetailsDetails about Dual Ascent}
%
\subsection{Dual Ascent for STP.}\label{app:dual}
%It starts by creating a graph $T$ with $k$ components, one for each terminal, consisting only of the terminal itself. In each iteration, one terminal is chosen and the component $C$ consisting of all vertices with a path to the terminal is computed. The cheapest incoming arc to $C$ that is not in $T$ is added to $T$ and its cost is subtracted from all other such incoming arcs. As soon as another component has a path to $C$, the respective terminal is regarded as inactive, and not used anymore. The algorithm thereby constructs a sub-graph that connects all terminals. Summing up the weights of all arc when they are added, yields a lower bound. The approximation is shown in Listing~\ref{alg:da}.

\begin{algorithm2e}[t]
\KwData{An instance~$(N,R)$ of STP, where $N=(V,E,\sigma)$, $R \subseteq V$ is a set of terminals, and $r \in R$ (``root'').}\KwResult{A cost lower bound~$\tilde{w}$ of any SMT of~$(N,R)$.}
\caption{The dual ascent algorithm for STP, cf., \cite[Chapter~3.1]{PajorUchoaWerneck18a}}\label{alg:da}
$\tilde{\sigma}((u,v)) \gets \sigma(\{u,v\})$ \quad for each $\{u,v\}\in E$ \;
$\tilde{w} \gets 0$ \;
$Q \gets R\setminus\{r\}$ \;

\While{$|Q| > 0$}{
	Choose $z^\prime$ in $Q$ \; \label{alg:da:choice}
    $C \gets cut(z^\prime)$ \;
    \uIf{$|(Q \cup \{r\}) \cap C| > 1$}{
    	$Q \gets Q \setminus \{z^\prime\}$\;
    }
    \Else{
    	$F \gets \{(u,v) \mid \{u,v\}\in E, v_i \notin C, v_j \in C\}$\;
    	$c^\prime \gets \min_{(u,v) \in F} \tilde{\sigma}((u,v))$ \;
        $\tilde{w} \gets \tilde{w} + c^\prime$ \;
        \ForEach{$(u,v) \in F$} {
        	$\tilde{\sigma}((u,v)) \gets \tilde{\sigma}((u,v)) - c^\prime$\;
        }
    }
}
\Return{$\tilde{w}$}
\end{algorithm2e}
The algorithm is shown in Listing~\ref{alg:da} and takes as argument an instance~$\mathcal{I}=(N,R)$ of STP, where $N=(V,E,\sigma)$, and~$r\in R$.
In the algorithm, $Q$ stores the active terminals, $\tilde{\sigma}$ the updated costs of (directed) edges, and $\tilde{w}$ keeps track of the lower bound. 
Function $cut$ computes for a given terminal~$z$ the corresponding subgraph by regarding every edge with $0$ costs as part of the subgraph. 
The algorithm runs in time $\mathcal{O}(|E| \cdot \min\{\Card{V} \cdot \Card{R}, |E|\})$ \cite{duin_steiners_1993}.

\subsection{Pruning for DS.}  \label{apdx:pruning}
Next, we discuss the pruning strategy for a given instance $(N, R)$, where~$N=(V,E)$, and~$r\in R$, as used in Lines~\ref{alg:dijkstra:prune2} and~\ref{alg:dijkstra:prune1} of Listing~\ref{alg:dijkstra:final}. We limit ourselves to the definition and refer the reader to the literature~\cite{HougardySilvanusVygen16a} for further details. 
Pruning works by comparing tentative costs $l(v, J)$ for a tuple~$(v,J)$ against an upper bound $U(J)$, which is maintained for pruning. 
%To this end, the solver maintains for~$J$ an upper bound $U(J)$. 
Whenever $l(v, J)$ exceeds $U(J)$ the $prune$ and $prune\_combine$ functions calls return \emph{true}, since exceeding this upper bound implies that $(v, J)$ is not a sub-SMT for the instance and is therefore disregarded by the algorithm.

Whenever $prune$ is called with arguments $v$ and $J$, it updates $U$. 
To this end, we define the \emph{set-distance} $sd: 2^V \times 2^V \rightarrow \mathbb{N}$ as $sd(A, B) \eqdef \min_{x \in A, y \in B} d_N(x,y)$.
 We then update $U(J) \leftarrow \min\{U(J), l(v, J) + \min\{sd(J, R \setminus J), sd(\{v\}, R \setminus J\}\}$. 
 Additionally, we maintain a set $S(J)$. Whenever $prune$ changes $U(J)$, we have then exactly one $z \in R \setminus J$ that was the reason for finding the minimum value using $sd$ in~$U(J)$. %\todo{why?}
 We then set $S(J) \leftarrow \{z\}$.
 
Function $prune\_combine$ is an extension of $prune$, that tries to combine upper bounds, before proceeding as described above. 
Whenever the method is called with arguments $J_1$ and $J_2$, it checks if $U(J_1)$ and $U(J_2)$ are defined. 
If this is the case and either $S(J_1) \cap J_2 = \emptyset$ or $S(J_2) \cap J_1 = \emptyset$, then the two upper bounds are combined. 
This is done by setting $U(J_1 \cup J_2) \leftarrow U(J_1) + U(J_2)$ and $S(J_1 \cup J_2) \leftarrow (S(J_1) \cup S(J_2)) \setminus (J_1 \cup J_2)$.

\section{Correctness Proof of \dsstar for Admissibility}\label{chap:proof}
%\todo{Maybe reintroduce notation? R, $R^\prime, v_i$, ...}
In this section, we consider again an instance~$\mathcal{I}=(N,R)$ of STP, 
where~$N=(V,E,\sigma)$, and $r\in R$.
Further, let~$h^*$ be an admissible, but not necessarily consistent Steiner heuristic function.
We prove that algorithm \dsstar as presented above in Listing~\ref{alg:dijkstra:final} is still correct for~$h^*$.
To simplify the proof, we show this without pruning, i.e., we assume that there is no pruning 
and therefore the if-conditionals
in Lines~\ref{alg:dijkstra:prune2} and~\ref{alg:dijkstra:prune1} of Listing~\ref{alg:dijkstra:final} are always true.
The proof of correctness including pruning can be established by showing that the invariants still hold. 
Intuitively, pruning uses tentative costs $l(u, I)$ of tuples~$(u,I)$ to compute an upper bound, if the value is non-optimal, which then mainly increases practical efficiency by still preserving validity.

We reuse the following results from the correctness proof~\cite{HougardySilvanusVygen16a} of DS, that do not depend on the consistency of~$h^*$.
\begin{enumerate}
    \item[R1] For each $(u, I) \in P \cup Q$, we have:
        \begin{enumerate}
            \item[R1.1] $b(u, I) \subseteq P$ and $\computeSMT(u, I)$ returns the set of edges of a connected graph $T$ such that all vertices $\{u\} \cup I$ are in $T$ and $c_N(T) \leq l(u, I)$, and
            \item[R1.2] $I \cup \{u\}$ = $\bigcup_{(v, J)\in b(u, I)} J \cup \{u\}$.
        \end{enumerate}
    \item[R2] For each $(u, I) \in (V \times 2^{[R\setminus\{r\}]}) \setminus P$, we have:
        \begin{enumerate}
            \item[R2.1] $l(u, I) \geq \csmt(N, \{u\} \cup I)$, and
            \item[R2.2] if $l(u, I) = \csmt(N, \{u\} \cup I)$, then $(u, I) \in Q$.
        \end{enumerate}
\end{enumerate}

\noindent
We also introduce a few definitions. We denote by $\retraceTree(u, I)$ the graph defined by the set $\computeSMT(u, I)$ of edges.
Furthermore we define a tuple $(u, I)$ to be \emph{optimal}, 
if $l(u, I) = \csmt(N, I \cup \{u\})$ and $\retraceTree(u, I)$ is a subgraph of at least one $T \in \smt(\mathcal{I})$. 
Note that the tentative costs of an optimal tuple are never ``overwritten'', since any new tuple cannot have lower weight by \emph{R1.1} and \emph{R2.1}. %\todo{why is this guaranteed? what is entry, what is overwriting?}

We define the transformation $tr(u, T_1, T_2)$ for two trees $T_1=(V_1,E_1,\cdot), T_2=(V_2,E_2,\cdot)$, where $T_1$ is a subgraph of $T_2$ and $u \in V_1$. 
Let $V' = V_2 \setminus (V_1 \setminus \{u\})$ and $E' = (E_2 \setminus E_1)$. 
The result of the transformation~$tr(u, T_1, T_2)$ is the graph $(V', E')$. 
Intuitively, the result of $tr$ is the graph obtained after removing from $T_2$ the edges in $T_1$. 

Furthermore, we define $tr_v(u, T_1, T_2)$ for trees~$T_1=(V_1,E_1,\cdot), T_2=(V_2,E_2,\cdot)$ s.t.\ $T_1$ is a subgraph of~$T_2$ and~$u\in V_1, r\in V_2$, as follows.
%analogously to $tr$, adding a vertex $r \in V_2$. 
Let therefore $T' = tr(u, T_1, T_2)$. 
Since~$T'$ is acyclic, there exists a unique path $S_{u, v}$ from $u$ to $v$.
Let $T'' = tr(u, S_{u, v}, T') = (V'', E'')$ and $C=(V_C, E_C)$ be the largest subgraph of $T''$ containing $u$. 
The result of the transformation $tr_v(u, T_1, T_2)$ is the graph $(V_C, E'' \cap (V_C \times V_C))$. 
After the transformation, only the connected subgraph containing $u$, without $T_1$, remains of $T_2$.

\begin{lemma}[Loop Invariant]\label{lemma:loop}
The following invariants hold at the beginning of every iteration of the main loop in Line~\ref{alg:dijksteiner-main-loop} of Listing~\ref{alg:dijkstra:final}.
There exists a partition $\sset{W}$ of set $R\setminus\{r\}$ of vertices such that:
\begin{itemize}
\item[H1] There exists $T \in \smt(\mathcal{I})$, such that
    \begin{itemize}
        \item[H1.1] for each $W \in \sset{W}$, there exists $(u, W) \in P \cup Q$ such that $\retraceTree(u, W)$ is a subgraph of $T$ and $(u, W)$ is optimal, and
        \item[H1.2] for each $(u, W) \in P \cup Q$, such that $\retraceTree(u, W)$ is a subgraph of $T$, there exists no $W' \in \sset{W}$ with $W \neq W'$ such that $(u, W \cup W')$ is optimal, $\retraceTree(u, W')$ is a subgraph of $T$ and $(u, W \cup W') \in P \cup Q$.
    \end{itemize}
    
\item[H2] For each $W \in \sset{W}$ we have
    \begin{itemize}
        \item[H2.1] there exists a tuple $(u, W) \in Q$ such that $(u, W)$ is optimal, or
        \item[H2.2] there exists a tuple $(u, W) \in P$ such that $(u, W)$ is optimal. 
	Furthermore, there exists $T \in \smt(\mathcal{I})$ such that $\retraceTree(u, W)$ is a subgraph of $T$ and $\Card{R \cap  V'} > 0$, where $(V',E')=tr_r(u, \retraceTree(u, W), T)$. 
    \end{itemize}
\end{itemize}
\end{lemma}

In the following, we refer to the invariant conditions by the corresponding label.
Before we prove the lemma, we prove some results of this invariant. 
To this end, we let for given~$T\in\smt(\mathcal{I})$ and tuple~$(u,I)$, $\phi(u) \vcentcolon=R \cap V'$, where $(V', E')\eqdef tr_r(u, \retraceTree(u, I), T)$. Note that~$\phi(u)$ is well-defined whenever, $\retraceTree(u, I)$ is a subgraph of~$T$.
\begin{lemma}\label{lemma:loop-it}
Assume that invariant~\emph{H1} and~\emph{H2} holds at the beginning of the current iteration. 
Then, let $\sset{W}$ be a partitioning of $R\setminus\{r\}$ as described in Lemma~\ref{lemma:loop}. Further, let $(u, I)$ be the tuple that is chosen from~$Q$ in the current iteration in Line~\ref{alg:dijkstra:heuristic} of Listing~\ref{alg:dijkstra:final}. 
If $(u, I)$ is optimal and $I \in \sset{W}$, then invariant \emph{H2} holds at the end of the current iteration as well.
\end{lemma}
\begin{proof}
In the following, we show that after the iteration there always exists a tuple $(v, I)$ such that \emph{H2.1} or \emph{H2.2} holds. %\todo{what does this mean?}
To this end, let $T \in \smt(\mathcal{I})$ be any SMT containing $\retraceTree(u, I)$ as a subgraph. Due to optimality of~$(u,I)$, such a graph exists.  Then, let $S_{u,r}$ be the unique path from $u$ to $r$ in $T$.

We prove the lemma by contradiction. We therefore assume that no tuple $(v, I)$ exists, such that \emph{H2} holds. 
Then, for all optimal tuples $(v^\prime, I)\in P$ we have that, $(v^\prime, I) \notin Q$ and $|\phi(v^\prime)| = 0$.
%, where $\phi(v^\prime) \vcentcolon=R \cap V'$ s.t.\ $(V', E')\eqdef tr_r(v^\prime, \retraceTree(v^\prime, I), T)$ (see \emph{H2.2}).  
%\todo{does this now hold, or do we want to show this?}  
%\todo{Is the degree in relation to the function, or really in $T$?}
Next, we aim to arrive at a contradiction by induction on $\ell$, where $x_\ell$ denotes the $\ell$-th vertex in $S_{u,r}$.  %\todo{which result?}
\begin{itemize}
    \item \emph{Induction Hypothesis}: For every vertex $v'$ of $S_{u,r}$ it holds that, $(v^\prime, I)$ is optimal, $(v^\prime, I) \in P, (v^\prime, I) \not\in Q$, and $|\phi(v^\prime)| = 0$.
    \item \emph{Base Case ($x_1$)}: $(u, I)$ is by definition optimal. In the current iteration $(u, I)$ is removed from $Q$ and added to $P$. 
    Furthermore, by assumption we already have that $|\phi(u)|=0$. %\todo{why?}
    \item \emph{Induction Step ($x_{\ell+1}$)}: 
    For vertex $x_{\ell+1}$ it holds by hypothesis that, $\Card{\phi(x_\ell)}=0$,
     %\todo{not sure. However, so?}
    %Since~$x_\ell$ has degree at most two, 
    %Further, by hypothesis 
    $(x_\ell, I) \notin Q$,
    $(x_\ell, I) \in P$ and $(x_\ell, I)$ is optimal.
    As a result, $(x_\ell, I)$ has been expanded with optimal costs. 
    Then, $\Card{\phi(x_\ell)}=0$ implies, that $x_\ell$ has degree at most two in $T$. 
    If it had degree 3 or higher, one of the subgraphs would contain no terminal, contradicting that $T$ is minimal.
    Due to the optimality of $(x_\ell, I)$ and since $x_\ell$ has degree 1 or 2 in~$T$, the optimal costs for $(x_{\ell+1}, I)$ are $l(x_\ell, I) + c_N(\edge{x_\ell}{x_{\ell+1}})$. 
    Therefore, after the expansion of $(x_\ell, I)$, the tuple $(x_{\ell+1}, I)$ was optimal. 
    By \emph{R2.2}, $(x_{\ell+1}, I)$ is therefore either in $Q$ or $P$. 
    Then, by assumption $(x_{\ell+1}, I) \in P$ and also by assumption $|\phi(x_{\ell+1})| = 0$.
\end{itemize}
Since $r \in S_{u,r}$ and $|\phi(r)| > 0$, we have a contradiction.
\end{proof}

We now have all the tools to prove Lemma~\ref{lemma:loop}:
\begin{proof}[of Lemma~\ref{lemma:loop}]
We show that the invariant holds at the beginning of the first iteration. 
Every tuple in $Q$ is of the form $(z, \{z\})$, where $z \in R\setminus\{r\}$. 
For each~$z\in R\setminus \{r\}$, these sets~$\{z\}$ induce a partition and $\retraceTree(z, \{z\})$ is a subgraph of any SMT of~$(N,R)$.
As a result, \emph{H1} and~\emph{H2} sustains.

We assume that we are at the beginning of an arbitrary iteration, and assume that the invariant holds for the previous iteration. Next, we remove a tuple $(u, I)$ from $Q$. Towards showing the invariant for the current iteration, the proof proceeds by a case distinction on $(u, I)$.

\emph{Case 1}: $(u, I)$ is not optimal. In this iteration, no optimal value can be ``overwritten'' in $l$, as the weight of any tuple processed in this iteration cannot be lower than an optimal value. 
As a result, the invariant still holds at the end of the iteration.

\emph{Case 2}: $(u, I)$ is optimal, but there exists no partitioning $\sset{W}$, such that $I \in \sset{W}$ and \emph{H1} and \emph{H2} hold for $\sset{W}$. However, by the invariant of the previous iteration, a partition~$\sset{W}$ exists such that~$I\not\in\sset{W}$. 
We show in the following, that neither loop in Line~\ref{alg:dijksteiner-nb-loop}, nor loop in Line~\ref{alg:dijksteiner-set-loop} invalidate the invariant.
%\todo{but we assumed that it does not?}

Observe, that the loop in Line~\ref{alg:dijksteiner-nb-loop} cannot overwrite tentative cost values of any tuples used in the invariant. 
In particular, no optimal value in $l$ can be overwritten, and no changes to $l$ and $b$ are performed in case of equality. 

%\todo{why now this, you said before, nothing can be changed.}
Assume that the loop in Line~\ref{alg:dijksteiner-set-loop} creates an optimal tuple $(u, I \cup J)$. 
This could invalidate the invariant only, if there exist $I^\prime, J^\prime \in \sset{W}$ such that $I \cup J = I^\prime \cup J^\prime$. Let $T$ be the SMT as defined in \emph{H1}. 
Then, $I \notin \sset{W}$ and therefore, either $I \not\subseteq I^\prime$ or $J \not\subseteq J^\prime$. 
This implies that either $\retraceTree(u, I)$ or $\retraceTree(u, J)$ is not a subraph of $T$. 
Therefore, \emph{H1.2} still holds, as well as the other invariants.

\emph{Case 3}: $(u, I)$ is optimal and there exists a partitioning $\sset{W}$ as described in the lemma. We only consider optimal tuples that are created in $Q$ or $P$ during the iteration. All other tuples do not affect any of the invariants. 

If no new optimal tuples in $Q$ are created, then, since invariant~\emph{H2} was true at the end of the previous iteration by Lemma~\ref{lemma:loop-it}, invariant \emph{H2} is still preserved. 
As $(u, I)$ is added to $P$, \emph{H1.1} holds. And as no new tuples are created, \emph{H1.2} holds as well.

For every optimal tuple created in the loop at Line~\ref{alg:dijksteiner-nb-loop} (propagation), obviously \emph{H2.1} and therefore \emph{H2} is preserved.
Since $(u, I)$ is added to $P$, \emph{H1.1} is also preserved. 
Finally, since all tuples are created using set $I$, \emph{H1.2} holds as well.

For every optimal tuple processed in the loop at Line~\ref{alg:dijksteiner-set-loop} (combination). Tuples created consist of set $I \cup J$. 
We proceed with a case distinction on $J$:
\begin{itemize}[itemsep=4pt]
\item[] \emph{Case 3a}: $J \notin \sset{W}$; \emph{H1} holds after the iteration as $(u, I) \in P$. \emph{H1.1} holds since $(u, I) \in P$ and \emph{H1.2} holds because $J \notin \sset{W}$. 
Since $(u, I \cup J)$ is optimal, there exists $T \in \smt(\mathcal{I})$ such that $\retraceTree(u, I \cup J), \retraceTree(u, I)$ and $\retraceTree(u, J)$ are subgraphs of~$T$. 
Since $\retraceTree(u, I)$ and $\retraceTree(u, J)$ are subgraphs of $\retraceTree(u, I \cup J)$, and each contains at least one terminal (since $I,J\neq \emptyset$), \emph{H2.2} holds.

\item[] \emph{Case 3b}: $J \in \sset{W}$; We argue that $\sset{W}^\prime = (\sset{W} \setminus \{I, J\}) \cup \{I \cup J\}$ is a partition preserving the invariant. 
Let $T$ be the SMT as described in \emph{H1} for $\sset{W}$. 
\emph{H1.1} holds since either $\retraceTree(u, I \cup J)$ is a subgraph of $T$, or one of $\retraceTree(u, I)$, $\retraceTree(u, J)$ is no subgraph of $T$. \emph{H1.2} sustains since it holds for the tuples~$(u,I)$, $(u,J)$ and therefore also for the tuple~$(u,I\cup J)$. %\todo{bissl fishy} 
Since at the beginning of the iteration $(u, I \cup J) \notin P \cup Q$ by \emph{H1.2}, $(u, I \cup J)$ is added to $Q$ and therefore \emph{2.1} and consequently \emph{H2} holds.%
\end{itemize}
\end{proof}

\begin{lemma}\label{lemma:loop-queue}
Whenever an element from $Q$ is chosen in the loop in Line~\ref{alg:dijkstra:heuristic} of Listing~\ref{alg:dijkstra:final}, there exists $(u, I) \in Q$ that is optimal.
\end{lemma}
\begin{proof}
%\todo{Maybe move the definition of phi up?}
Let $\sset{W}$ be a partition as described in Lemma~\ref{lemma:loop}. We first show the following Claim (1) that given $T$ as defined in \emph{H1},
%and $\phi(v) \vcentcolon=R \cap V[tr_r(v, \retraceTree(v, I), T)]$, it holds
we have that for every $W \in \sset{W}$, there exists an optimal tuple $(u, W) \in P$ such that $|\phi(u)| > 0$ and $\retraceTree(u, W)$ is a subgraph of $T$. %Note that this is different from \emph{H2.2} as $T$ is a specific SMT. \todo{tollpatschig}
Let therefore $W \in \sset{W}$ be an arbitrary element and $(u, W)$ be an optimal tuple such that $\retraceTree(u, W)$ is a subgraph of $T$. By \emph{H1}, such a tuple exists. 
Then, we establish \emph{H2.2} for tuple~$(u,W)$, by using the same proof as in Lemma~\ref{lemma:loop-it}, but for specific tuple~$(u,W)$.
This concludes Claim (1). %\todo{cannot be the same proof, which property?}

Next, we prove the lemma by contradiction. 
We therefore assume that no tuple $(u, I) \in Q$ exists that is optimal. 
From Claim (1) above, we follow that for every $W \in \sset{W}$, there exists a tuple, satisfying \emph{H2.2}, which is a subgraph of $T$.
As each of the tuples defines a subgraph in $T$, we have $|\sset{W}|$ many subgraphs. 
There can be at most $|\sset{W}|-1$ many join vertices joining these subgraphs, since~$T$ is a tree.
We define for any $u \in V$ the set $A(u) = \{W \in \sset{W} \mid (u, W) \textrm{ is optimal}\}$ of sets in $\sset{W}$ constituting optimal tuples. 
Since we have $|\sset{W}|$ many subgraphs connected by at most $|\sset{W}| - 1$ join vertices, there exists $u \in V$ such that $|A(u)| > 1$ by basic combinatorics. 
Since all tuples $(u, W)$ with $W \in A(u)$ are optimal, the tuple $(u, \bigcup_{W\in A(u)} W)$ would have been added to $Q$ with optimal value in Line~\ref{alg:dijksteiner-add2}. This contradicts \emph{H1.2}.
\end{proof}

\noindent
We are in the situation to prove our main result, Theorem~\ref{thm:main}.

\begin{restatetheorem}[thm:main]
\begin{Theorem}
Given an instance~$(N,R)$ of STP, where~$N=(V,E,\sigma)$, and a root~$r\in R$. 
Then, Algorithm~\dsstar, cf., Listing~\ref{alg:dijkstra:final}, (1) terminates if $h^*$ is admissible. 
Further, after termination, (2) $l(r, R \setminus\{r\})$ is the cost of an SMT for the instance $(N, R)$, i.e., $l(r, R \setminus\{r\})=l^*(r, R\setminus\{r\}=\csmt(N,R)$. 
\end{Theorem}
\end{restatetheorem}
\begin{proof}
For showing termination (1), consider that any value $l(u,I)$ that is set is clearly bounded above by the sum of all edges. 
Assuming~$n\eqdef \Card{V}$, the Steiner network has $2^n \cdot n$ tuples. 
Since creating a new tuple $(u, I)$ in $Q$ requires that the found cost is lower than $l(u, I)$, therefore only $2^n \cdot n \cdot c_N((V, E))$ tuples can be created. 
Since this number is finite, the algorithm terminates.

For establishing correctness (2), we have to show that whenever $(r, R\setminus \{r\}) \in P$, then $l(r, R\setminus\{r\}) = \csmt(N, R)$. 
Towards a proof by contradiction, we assume that $(r, R\setminus\{r\}) \in P$ and $l(r, R\setminus\{r\})$ is not equal to the optimal weight. 
Since Steiner heuristic function $h^*$ is admissible, we have $h^*(v, R \setminus I) \leq l^*(v, R \setminus I) \textrm{ for all }  v \in V, \{r\} \subseteq I \subseteq R$.
Let $w = \csmt(N, R)$. Due to the admissibility of~$h^*$, it holds that $w' \eqdef l(r, R\setminus\{r\}) + h^*(r, \{r\}) = l(r, R\setminus\{r\})$. Since~$(r, R\setminus\{r\})$ is not optimal by assumption, $w' > w$. 
Since tuple~$(r, R\setminus\{r\})$ was expanded, it holds that at this iteration, no tuple $(u, I) \in Q$ such that $l(u, I) + h^*(u, R \setminus I) \leq w$ existed. 
Any optimal tuple $(v, J)$ would have $l(v, J) + h^*(v, R \setminus J) \leq l^*(v, J) + l^*(v, R \setminus J) = w$ due to optimality and the admissibility of the Steiner heuristic function. 
Therefore, there existed no optimal tuple in $Q$, contradicting Lemma~\ref{lemma:loop-queue}.
\end{proof}

\section{Details about Preprocessing in \dsstarsolver}
%\subsection{Preprocessing Module.}  
\label{apdx:preprocessing}
In this section, we give a description of all the preprocessing techniques mentioned in Section~\ref{chap:preprocessing}. 
We first introduce some concepts used in preprocessing and then discuss the different preprocessing operations. 
We assume in the following an instance $\mathcal{I} = (N, R)$, where $N = (V, E, \sigma)$, and~$G=(V,E)$.

\paragraph*{New Edges.}
It is common that a reduction introduces an edge $e$ with edge costs $c_e$. 
Whenever $e$ is already part of $N$ we then retain the minimum of the costs, i.e., $\sigma(e) \leftarrow \min\{c_e, \sigma(e)\}$. Otherwise, we set~$\sigma(e)\eqdef c_e$.
%For this section we implicitly assume, that an edge is identified not only by its endpoints, but also by its costs.

\paragraph*{Contraction.}
%A transformation then is applied by preprocessing operations, is the \emph{contraction} of an edge. 
An edge $\{u,v\}$ is contracted by performing the following steps:
\begin{enumerate}
\item For each edge $\{v, v^\prime\} \in \delta_G(v) \setminus \{u, v\}$, add edge $\{u, v^\prime\}$ to $N$, with edge costs $\sigma(\{v, v^\prime\})$.
\item If $v \in R$ then $R \leftarrow (R \setminus \{v\}) \cup \{u\}$.
\item Remove vertex $v$ and edges $\delta_G(v)$ from $N$.
\end{enumerate}

\paragraph*{Bottleneck Steiner distance.}
%Additionally to the regular distance $d_N(u, v)$, we define in the following the \emph{bottleneck Steiner distance}~$s_\mathcal{I}(u, v)$.
%
%To this end, we need the \emph{Steiner distance} between two vertices. 
Given a path $P$ between two vertices $u$ and $v$. 
Then, this path can be transformed into one or more \emph{elementary paths} by splitting it at the intermediate terminals along~$P$. 
The maximum costs~$c_N(P')$ among each~$P'$ of these elementary paths is the \emph{Steiner distance}.

The \emph{bottleneck Steiner distance} between vertices $u$ and $v$, denoted by $s_\mathcal{I}(u, v)$, is the minimum \emph{Steiner distance} among all paths connecting $u$ and $v$.
The \emph{restricted bottleneck Steiner distance} $\overline{s}_\mathcal{I}(u, v)$ is the minimum Steiner distance among all paths connecting $u$ and $v$ without using the edge $\{u, v\}$. 
Observe that if $\{u, v\} \notin E$ then $\overline{s}_\mathcal{I}(u, v) = s_\mathcal{I}(u, v)$  \cite{duin_reduction_1989}.

The values for $s_\mathcal{I}$ and $\overline{s}_\mathcal{I}$ can be computed exactly \cite{HwangRichardsWinter92a} or approximated \cite{duin_steiners_1993, polzin_tobias_algorithms_2003}. 
As computing the exact values might not pay off for large instances, 
we oftentimes use different approximations.
%
%As we assume a specific $\mathcal{I}$ for this chapter, we omit the subscript.

%%%%%%%%%%%%%%%%%%%%%%%%%%%%%%%%%%%%%%%%%%%%%%%%%%%%%%%%%%%%%%%%%%%%%%%%%%%%%%%%%%%%%%
\subsection{Preprocessing Operations.}
The \dsstarsolver solver uses a collection of preprocessing operations,
as listed below in three categories.
%We seperate these into three sets:
\begin{enumerate}
\item \emph{Simple operations} are fast and can therefore be run, whenever the instance is simplified. %, see~\cite[Chapter~2.1.1, 2.2.1, 2.2.2]{hwang_reductions_1992} %changes.
\begin{enumerate}
	\item \emph{non-terminals of degree $\leq 1$} \cite[Chapter~2.1.1]{HwangRichardsWinter92a}
	\item \emph{non-terminals of degree~2} \cite[Chapter~2.2.2]{HwangRichardsWinter92a}
	\item \emph{terminals of degree 1} \cite[Chapter~2.2.1]{HwangRichardsWinter92a}
	\item \emph{Include minimum terminal edge} \cite[Chapter~2.2.2]{HwangRichardsWinter92a}
	\end{enumerate}
\item \emph{Exclusions}, which try to remove a sub-graph by showing that there exists an SMT without the sub-graph.
	\begin{enumerate}
	\item \emph{Long Edges} \cite[Lemma~20]{polzin_tobias_algorithms_2003}
	\item \emph{Steiner Distance} \cite[Chapter~3]{duin_reduction_1989}
	\item \emph{Non-Terminals of degree $k$} \cite[Chapter~4]{duin_reduction_1989}
	\item \emph{Dual ascent bounds} \cite[Chapter~3.4.2]{polzin_tobias_algorithms_2003}
	\end{enumerate}
\item \emph{Inclusions} try to merge a sub-graph with an other sub-graph by showing that there exists an SMT containing this sub-graph.
	\begin{enumerate}
	\item \emph{Short Links} \cite[Lemma~23]{polzin_tobias_algorithms_2003}
	\item \emph{Nearest Vertex} \cite[Chapter~2.2.3]{HwangRichardsWinter92a}
	\end{enumerate}
\end{enumerate}
The preprocessing operations are run in the order presented above, where the simple operations are also applied in-between.
For each operation, we define a threshold in terms of a vertex or edge ratio. Whenever an operation fails to remove the number of vertices or edges defined by its threshold, it is deactivated. As soon as no active operations are left, the instance is passed on to the solving algorithm. The thresholds avoid expensive computations with negligible results.

We discuss the different preprocessing operations used in our solver.
Every preprocessing operation consists of a lemma indicating a ``test'' of applicability presented, and a transformation. 
Whenever the test identifies a vertex or edge, the transformation can be performed. 
After obtaining an SMT for the resulting preprocessed instance, these transformations are later reversed on the SMT. 
This yields a solution for the original instance.

\paragraph*{Remove Non-Terminals of degree 0 or 1.}
Finding non-terminal leaf vertices, which can be removed from $N$,
can be done extremely fast. 
%Although it is generalized in subsequent tests, the fast runtime makes it invaluable.

\begin{lemma}\cite[Chapter~2.1.1]{HwangRichardsWinter92a}
No vertex $u \in V, u \notin R$ of degree 1 or 0 can be in a Steiner minimal tree.
\end{lemma}

\paragraph*{Connect Neighbors of Non-Terminals of degree 2.}\label{chap:theory-reduction-ntd2}
This operation makes use of an important property of non-terminals of degree two. Whenever any such vertex is in an SMT, both its neighbors have to be as well. Non-terminals of degree two can thereby be replaced by an edge connecting the two neighbors, as defined in the lemma.

\begin{lemma}\cite[Chapter~2.2.2]{HwangRichardsWinter92a}
Given a vertex $u \in V, u \notin R$ of degree two. Let $\{u, v\}, \{u, v^\prime\}$ be the two incident edges. 
Let $N^\prime$ be the network obtained by removing $u$ and adding edge $\{v^\prime, v\}$ with costs $\sigma(\{u, v\}) + \sigma(\{u, v^\prime\})$. Then, $\csmt(N, R) = \csmt(N^\prime, R)$. 
If $N$ already contains $\{v, v^\prime\}$ with $\sigma(\{u, v\}) + \sigma(\{u, v^\prime\}) \geq \sigma(\{v, v^\prime\})$, then $u$ can be removed.
\end{lemma}

\paragraph*{Include Edge of Terminals of degree 1.}
This simple operation allows us to quickly identify edges that must be in any SMT and can therefore be contracted. This similarly also applies to edges of cost 0.

\begin{lemma} \cite[Chapter~2.2.1]{HwangRichardsWinter92a}
For any terminal $z \in R$ of degree one, the incident edge $\{z, u\}$ is in every SMT.
\end{lemma}

\paragraph*{Include Minimum Terminal Edge.}
The following lemma is a generalization of the previous lemma and allows for more edges to be contracted. 

\begin{lemma} \cite[Chapter~2.2.2]{HwangRichardsWinter92a}
If the nearest vertex incident to a terminal $z$ is also a terminal $z'$, the edge $\{z, z^\prime\}$ is in at least one SMT.
\end{lemma}

\paragraph*{Steiner Distance.}
This operation makes use of the previously introduced \emph{bottleneck Steiner distance} to identify removable edges. We perform it with different approximation methods for $s_{\mathcal{I}}$ and $\overline{s}_{\mathcal{I}}$.

\begin{lemma}\cite[Chapter~3]{duin_reduction_1989}
Any edge $\{u, v\}$ with $\sigma(\{u, v\}) > s_{\mathcal{I}}(u, v)$ can be removed from $N$. The same holds for any edge $\{u, v\}$ such that $\sigma(\{u, v\}) \geq \overline{s}_{\mathcal{I}}(u, v)$ \cite{polzin_tobias_algorithms_2003}.
\end{lemma}

\paragraph*{Remove Long Edges.}
This operation complements the previous one by considering removable edges that may have been missed.

\begin{lemma}\cite[Lemma~20]{polzin_tobias_algorithms_2003}
Let $M$ be the MST for the distance network $D_N(R)$. Furthermore, let $cmax$ be the maximum cost among all edges in $M$. Any edge $\{u,v\}$ with $\sigma(\{u, v\}) > cmax$ can be removed from $N$.
\end{lemma}

\paragraph*{Remove Non-Terminals of degree~$k$.}
The goal of this operation is to identify non-terminals, that cannot have a degree higher than two in an SMT. 
To this end, given a set~$S\subseteq V$ of vertices, we define the \emph{Steiner distance network} $D^\prime(V) \defeq (V, V \times V, s^\prime)$, where $s^\prime(\{u, v\}) \eqdef s_{\mathcal{I}}(u, v)$. 
Further, let $M^\prime(V)$ be the MST of $D^\prime(V)$,
and $c'_{M^\prime}(V)\eqdef c_{D^\prime(V)}({M^\prime}(V))$ be the sum of edge costs of $M^\prime(V)$.

\begin{lemma}\cite[Chapter~4]{duin_reduction_1989}
For $u \in V \setminus R$ with $|\delta_G(u)| \geq 3$ exists an SMT where $u$ has maximum degree two if: For every $A \subseteq \delta_G(u)$ with $|A| \geq 3$, it holds that
\[
\sum_{\{u, v\} \in A} \sigma(\{u, v\}) \geq c'_{M^\prime}(\{v \mid \{u, v\} \in A \})
\]
\end{lemma}

Given such a vertex $u$ and its neighbors $V_A = \{v \mid \{u, v\} \in \delta_G(u) \}$, we can remove $u$ and replace it by edges connecting its neighbors. 
We add edges $\{(v, v^\prime) \in V_A \times V_A \mid v \neq v^\prime\}$ with costs $\sigma(\{v, v^\prime\}) \eqdef  \sigma(\{u, v\}) + \sigma(\{u, v^\prime\})$ for $v, v^\prime \in V_A, v \neq v^\prime$.

\paragraph*{Dual ascent bounds.}
We already discussed the dual ascent algorithm in Appendix~\ref{app:dual}. Besides a lower bound~$\tilde{w}$, the algorithm also maintains reduced arc costs $\tilde{\sigma}$. We discuss how to use these results to establish local lower bounds. Let $r$ denote the terminal that is designated as root, and let~$R'\eqdef R\setminus\{r\}$.

\begin{lemma}\cite[Chapter~3.4.2]{polzin_tobias_algorithms_2003}
Let $\tilde{w}$ be the lower bound and $\tilde{c}$ be the reduced edge costs after a dual ascent run. Furthermore, let $\tilde{d}$ be the distance function using the reduced edge costs and directed paths. Given a vertex $u \in V$, 
\[
\tilde{w} + \tilde{d}(r, u) + \min_{z \in R^\prime} \tilde{d}(u, z) 
\]
is a lower bound on any Steiner tree containing $u$. 

A similar bound can be derived for any edge $\{u, v\} \in  E$. Let
\begin{align*} 
c_u &= \tilde{d}(r, u) + \tilde{\sigma}((u, v)) + \min_{z \in R^\prime}  \tilde{d}(v, z)\\
c_v &= \tilde{d}(r, v) + \tilde{\sigma}((v, u)) + \min_{z \in R^\prime}  \tilde{d}(u, z)
\end{align*}
$\tilde{w} + \min\{c_u, c_v\}$ is a lower bound for any Steiner tree containing $\{u, v\}$.
\end{lemma}

%Note that the lower bound for an edge $\{u, v\}$ is computed by taking the minimum over the corresponding arcs $(u, v), (v,u)$. Furthermore, the lower bound on terminals is always 0.

Any vertex and edge, where this lower bound exceeds an upper bound can be removed from $N$. 
The upper bound can be computed using RSPH.

\paragraph*{Short Links.}
This operation uses the concept of \emph{Voronoi regions}. A Voronoi partitioning for $R$, separates the vertices of $N$ into $|R|$ disjoint sets, the Voronoi regions. For each terminal $z \in R$ we define the neighborhood $B(z)$, as those vertices $u \in V$  closer to $z$ than to any other terminal \cite{polzin_tobias_algorithms_2003}. Ties are assumed to be broken randomly. More precisely
\[
u \in B(z) \Rightarrow d_N(u, z) \leq d_N(u, z^\prime) \mbox{ for all } z^\prime \in R
\]
We say $base(u) = z$ iff $u \in B(z)$. Links are edges that connect two Voronoi regions, i.e., an edge $\{u, v\} \in E$ is a link iff $base(u) \neq base(v)$\cite{polzin_tobias_algorithms_2003}.

Any link identified by the following lemma, is a contractible edge.

\begin{lemma}\cite[Lemma~23]{polzin_tobias_algorithms_2003}
Given a terminal $z$ and the two shortest links of its Voronoi region $\{u, v\}, \{u^\prime, v^\prime\}$ such that $base(u) = base(u^\prime) = z$. $\{u, v\}$ belongs to at least one SMT if \[\sigma(\{u^\prime, v^\prime\}) \geq d_N(z, u) + \sigma(\{u, v\}) + d_N(v, base(v))
\]
\end{lemma}

\paragraph*{Nearest Vertex.}
This operation tries to find contractible edges incident to terminals. 
The idea is the following. If the shortest edge is short enough compared to the other edges, it is in an SMT.

\begin{lemma}\cite[Chapter~2.2.3]{HwangRichardsWinter92a}
Let $z$ be a terminal with degree at least 2. Let $e^\prime = \edge{z}{u}$, $e^{\prime\prime} = \edge{z}{v}$ be a shortest and second shortest edge incident to $z$. There exists at least one SMT containing $e^\prime$, if there exists a terminal $z^\prime$ such that $z^\prime \neq z$ and
\[
\sigma(e^{\prime\prime}) \geq \sigma(e^\prime) + d_N(u, z^\prime)
\]
\end{lemma}

In case the lemma fails to identify certain edges, we can extend it. 
If we can show that the condition holds for any edge that connects $v$, this is equal to showing it for $e^{\prime\prime}$:
\begin{lemma} \cite[Lemma~8]{Rehfeldt15a} %\FIX{This costs us a reference}
Given the definitions of the previous lemma, there exists at least one SMT containing $e^\prime$ if:
\[
\sigma(e) \geq \sigma(e^\prime) + d_N(u, z^\prime) \quad \mbox{for all } e \in \delta_G(z) \cup \delta_G(v) \setminus \{e^\prime, e^{\prime\prime}\}
\]
\end{lemma}

%\section{Details about Solving in \dsstarsolver}

\end{appendix}
\end{document}